\documentclass[lettersize,journal]{IEEEtran}
\usepackage{amsmath,amsfonts}
\usepackage{algorithm}
\usepackage{array}
\usepackage[caption=false,font=normalsize,labelfont=sf,textfont=sf]{subfig}
\usepackage{textcomp}
\usepackage{stfloats}
\usepackage{url}
\usepackage{verbatim}
\usepackage{graphicx}
\usepackage{cite}
\usepackage{hyperref}
\usepackage{booktabs}       %
\usepackage{amsfonts}       %
\usepackage{nicefrac}       %
\usepackage{microtype}      %
\usepackage{algorithmicx}
\usepackage{wrapfig}
\usepackage{booktabs}
\usepackage{algpseudocode}
\usepackage{diagbox}
\usepackage{comment}
\usepackage{multirow}
\usepackage{multicol}
\usepackage{subcaption}
\usepackage{caption}
\usepackage{graphicx}
\usepackage{float}
\usepackage{thm-restate}
\hypersetup{colorlinks = true,
            linkcolor = brown,
            urlcolor  = black,
            citecolor = brown,
            anchorcolor = brown}
\usepackage[framemethod=tikz]{mdframed}
\usepackage{lipsum}

\definecolor{mycolor}{rgb}{0.122, 0.435, 0.698}
\usepackage[most]{tcolorbox}

\usepackage{colortbl}
\usepackage{amsmath,amstext,amsfonts,bm,amssymb,mathtools,amsthm}
\usepackage{tcolorbox}
\usepackage{color,xcolor,xspace}
\usepackage{bbding}
\usepackage{pifont}
\usepackage{wasysym}
\usepackage{tikz}

\usepackage[inline]{enumitem}
\theoremstyle{plain}
\newtheorem{theorem}{Theorem}[section]

\newtheorem{lemma}[theorem]{Lemma}

\theoremstyle{definition}

\theoremstyle{remark}
\newtheorem{remark}[theorem]{Remark}
\usepackage[textsize=tiny]{todonotes}

\definecolor{shadecolor}{rgb}{0.92,0.92,0.92}
\usepackage{framed}

\usepackage{enumitem}

\begin{document}

\setlength {\marginparwidth }{2cm} 

\title{ZeroDiff++: Substantial Unseen Visual-semantic Correlation in Zero-shot Learning}

\author{Zihan Ye,~\IEEEmembership{Member,~IEEE,}, \thanks{$^\dagger$: Corresponding author: Ling Shao (ling.shao@ieee.org).\\
Zihan Ye and Ling Shao are with the UCAS-Terminus AI Lab, University of Chinese Academy of Sciences, Beijing 100049, China.\\
Shreyank N Gowda is with the School of Computer Science, the University of Nottingham, NG8 1BB Nottingham, UK. \\
Kaile Du is with the Southeast University, Nanjing 210096, China. \\
Weijian Luo is with the hi-lab of Xiaohongshu Inc, Beijing, China.
}
Shreyank N Gowda,
Kaile Du,
Weijian Luo,
Ling Shao$^\dagger$, ~\IEEEmembership{Fellow,~IEEE}
}



\maketitle

\begin{abstract}
Zero-shot Learning (ZSL) enables classifiers to recognize classes unseen during training, commonly via generative two stage methods: (1) learn visual semantic correlations from seen classes; (2) synthesize unseen class features from semantics to train classifiers. In this paper, we identify spurious visual semantic correlations in existing generative ZSL worsened by scarce seen class samples and introduce two metrics to quantify spuriousness for seen and unseen classes. Furthermore, we point out a more critical bottleneck: existing unadaptive fully noised generators produce features disconnected from real test samples, which also leads to the spurious correlation. To enhance the visual-semantic correlations on both seen and unseen classes, we propose ZeroDiff++, a diffusion-based generative framework. In training, ZeroDiff++ uses (i) diffusion augmentation to produce diverse noised samples, (ii) supervised contrastive (SC) representations for instance level semantics, and (iii) multi view discriminators with Wasserstein mutual learning to assess generated features. At generation time, we introduce (iv) Diffusion-based Test time Adaptation (DiffTTA) to adapt the generator using pseudo label reconstruction, and (v) Diffusion-based Test time Generation (DiffGen) to trace the diffusion denoising path and produce partially synthesized features that connect real and generated data, and mitigates data scarcity further. Extensive experiments on three ZSL benchmarks demonstrate that ZeroDiff++ not only achieves significant improvements over existing ZSL methods but also maintains robust performance even with scarce training data. Code would be available.
\end{abstract}


\section{Introduction}

\begin{figure*}[htbp]
\centering
\includegraphics[width=0.95\linewidth]{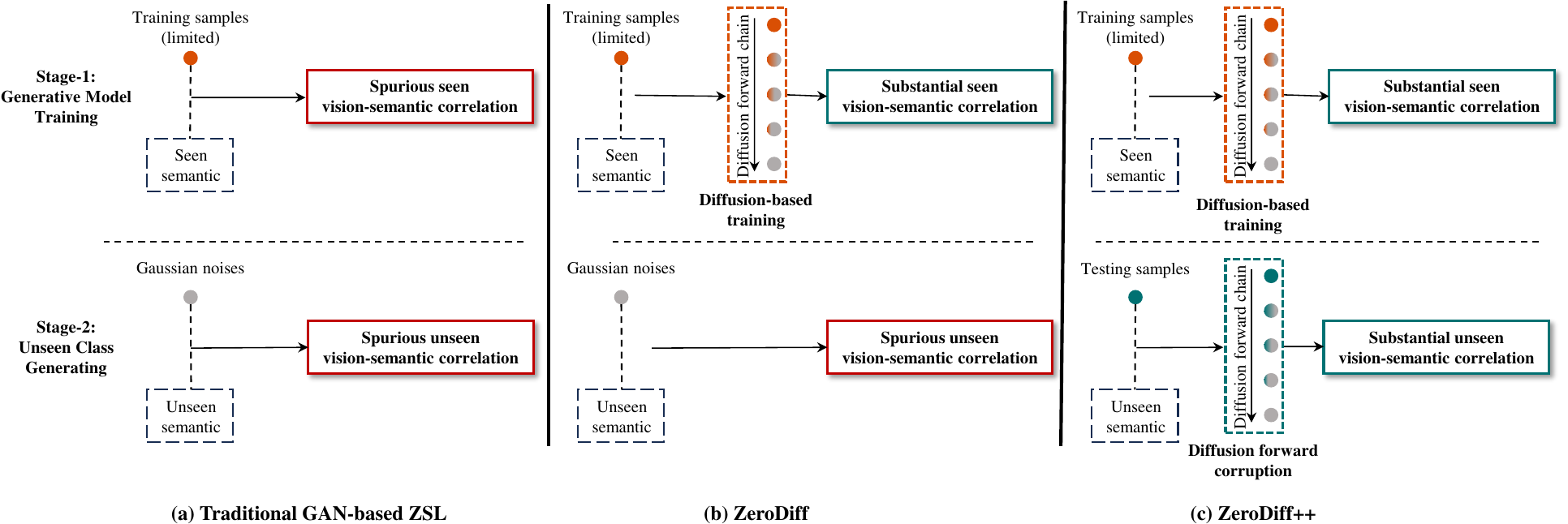}
\caption{Overall motivation illustration. (a) Traditional GAN-based ZSL methods suffer from spurious visual-semantic correlation both for seen and unseen classes. (b) ZeroDiff~\cite{ye2025zerodiff} employs a diffusion mechanism on the training stage to obtain a substantial correlation on seen classes. (c) Our ZeroDiff++ further explores the diffusion forward chain for utilizing real testing samples, encouraging substantial visual-semantic correlation on unseen classes.
}
\label{fig:banner}
\end{figure*}

\begin{figure*}
\centering
\includegraphics[width=0.95\linewidth]{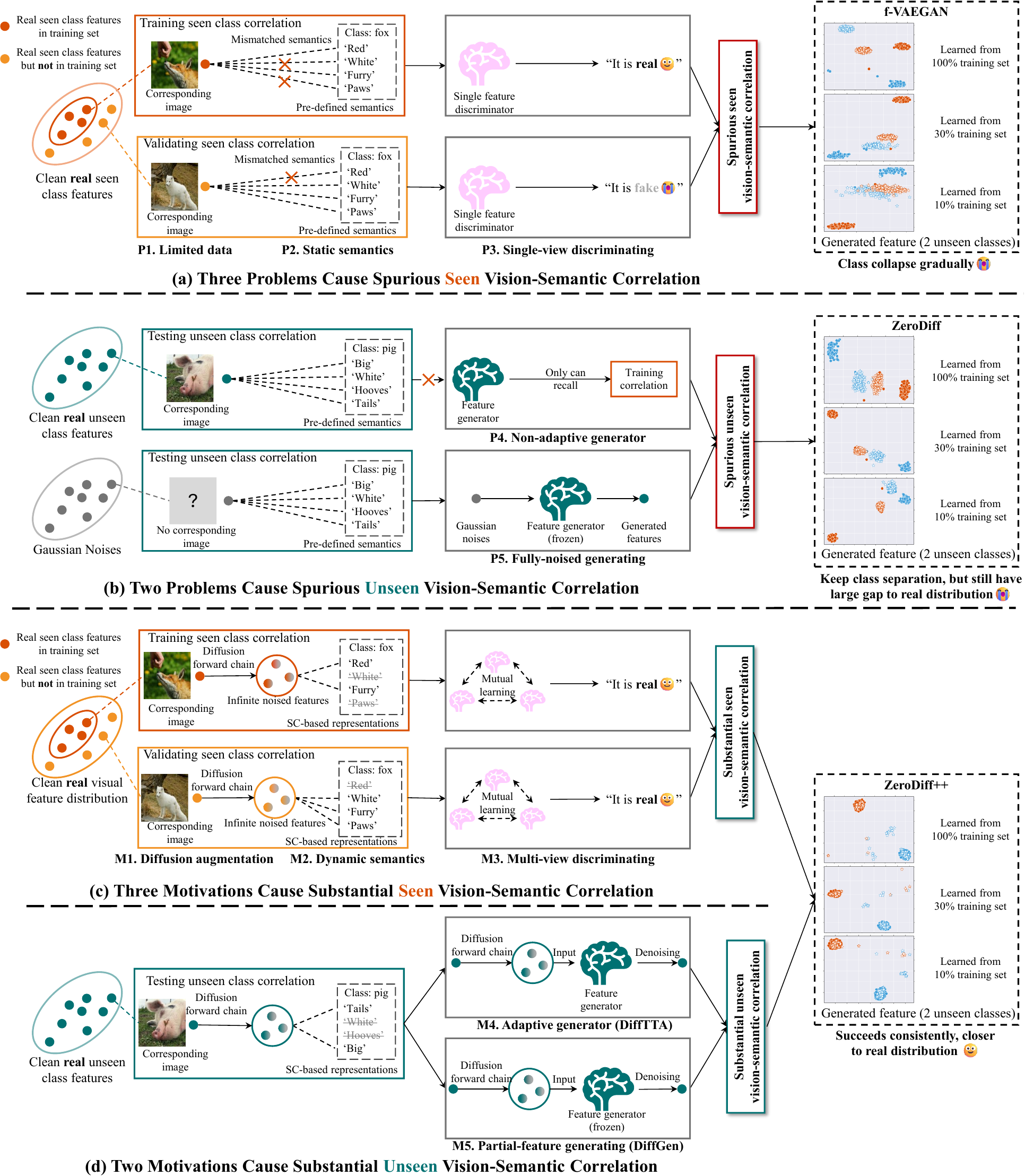}
\caption{Detailed motivation Illustration.
(a) In the training stage, three problems lead to the standard GAN-based ZSL approaches obtaining spurious seen correlation: over-fitting to limited data, mismatched static pre-defined semantics, and single-view discriminating.
(b) In the generating stage, two problems corrupt unseen correlation: an unadaptive generator and fully-noised generation.
Finally, spurious seen and unseen correlations lead to feature generation failing gradually.
(c) In contrast, our ZeroDiff++ overcomes these shortcomings using three motivations in the training stage for seen correlation: diffusion-augmented infinite features, dynamic SC-based representations, and multi-view discriminating.
(d) Two motivations in the generating stage are for unseen correlation: diffusion-based test-time adaptation and generating.
Finally, substantial correlations on both seen and unseen classes allow ZeroDiff++ to keep a robust performance with even 10\% training set.
}
\label{fig:motivation}
\end{figure*}

With the rapid development of machine learning, image recognition models have achieved unprecedented success, largely due to the availability of abundant labeled samples.
However, collecting extensive labeled datasets is often time-consuming and expensive, making it unrealistic to assume access to substantial volumes of labeled data.
To improve data efficiency for new classes, zero-shot learning (ZSL)~\cite{xian2018zero} offers a promising solution by transferring knowledge from seen classes to unseen classes through the use of pre-defined class semantic knowledge, such as attributes~\cite{chao2016empirical} and text-based representations~\cite{zhu2018generative, chen2021text}.

Conventional ZSL methods usually learn a visual-to-semantic embedding~\cite{fu2017zero}, a semantic-to-visual embedding~\cite{Fei_2021_ICCV}, or an unified embedding space ~\cite{Zhang_2015_ICCV}.
However, since only training samples of the seen class are available, such embedding methods inevitably tend to lean towards the seen class.
Under the more challenging Generalized ZSL (GZSL) task~\cite{chao2016empirical}, in which the model needs to classify samples of both seen and unseen classes simultaneously, the performance of the embedding method deteriorates.
This will make the embedding ZSL methods prone to misclassifying the unseen class images into seen classes and reduce the overall accuracy.

In contrast, to alleviate the biased prediction issue mentioned above, recent works try to utilize generative models to synthesize pseudo features of unseen classes, hence known as generative ZSL methods.
These methods typically leverage some form of adaptations of generative adversarial networks (GANs)~\cite{han2021contrastive,gowda2023synthetic, hou2024visual} to aid the feature generation.
Such a generic two-stage framework are wide-used:
(S-1) \textbf{Generative Model Training:} Training generative models to synthesize visual features from class semantics (i.e., the correlations between visual features and class semantics) on seen classes; (S-2) \textbf{Unseen Class Generating:} Then, freezing generative models to generate visual features of unseen classes by their semantics, and training an unseen class classifier by generated unseen class features.

In this paper, despite fast advances in ZSL, we point out that existing generative ZSL methods are limited by \textbf{spurious correlations on both seen and unseen classes}, as shown in Fig.~\ref{fig:banner} (a).
Specifically, for seen correlations, existing generative methods generally presume that there is a substantial number of samples available for seen classes to train generative models.
Only a few ZSL methods consider the data-limited ZSL~\cite{verma2020meta, tang2024data}.
Through empirical investigations, we observe that most existing generative ZSL methods suffer from performance degradation and collapsed generation modes as the number of training samples is gradually reduced.
A more detailed illustration is in Fig.~\ref{fig:motivation}(a). We re-split the real seen class samples into two groups: one for training and the other for validating the correlation. We then feed the visual features of the split samples into the GAN discriminator to obtain critic scores (i.e., the output of the discriminator), which indicate whether the discriminator perceives the input features as real or fake with respect to the class semantics. Next, we calculate the difference in critic scores as a measure to determine whether the learned visual-semantic correlation is spurious or substantial. Our observations reveal  that as the number of training samples decreases, the critic score difference becomes increasingly larger, suggesting that the generative models perceive the validating seen class samples as progressively more `fake.' This phenomenon indicates that a limited number of training samples amplifies the spurious seen visual-semantic correlation.

On the other hand, in the generating stage, we can similarly use the discrimination gap to indicate the unseen correlation. However, since only the generator is involved in this stage, the significance of the discrimination gap is reduced.
Thus, we further point out that a more critical factor is the lack of correlation between the generated samples and the real test samples, as shown in Figure~\ref{fig:motivation}(b). One reason is that the generator is only trained on seen classes and therefore cannot adapt to unseen classes.
Another reason is that existing methods typically sample noise from a Gaussian distribution and concatenate the sampled noise with the semantics of unseen classes to generate features for unseen classes.
This generation phase, which uses the semantics of unseen classes as a generation condition, can generate unseen class features with good separability.
However, good class separability does not guarantee that the generated features are consistent with the real features.
Due to the lack of correlation between real and generated features, the generated features may still be inconsistent with the semantic information of unseen classes.

To this end, we propose a novel ZSL framework mainly based on diffusion, which is named \textbf{ZeroDiff++}, for strengthening both seen and unseen vision-semantic correlation.
An overall illustration is in Fig.~\ref{fig:banner} (c).
Specifically, for seen correlation, three key insights are proposed onthe  training stage (Fig.~\ref{fig:motivation}(c)):
(1) \textbf{Diffusion augmentation}: Limited training samples can be easily memorized by models. We incorporate the diffusion mechanism~\cite{song2020denoising, wang2023patch} into our method, which allows a single clean sample to be augmented into an infinite number of noised samples by varying the noise-to-data ratios.
(2) \textbf{Dynamic semantics}: Predefined semantics are static, meaning that class semantics remain the same across different instances. However, each limited sample may only represent part of the predefined semantics. For example, in the AWA2 dataset, all images of the `fox' class are labeled with the semantics `red' and `white', even though this is inaccurate for white foxes. To address this, we revisit the classical Supervised Contrastive (SC) loss~\cite{khosla2020supervised} and suggest that SC-based representations can generate instance-level semantics for every sample, enhancing the generation of visual features.
(3) \textbf{Multi-view discriminating}: We combine three types of discriminators to assess the authenticity of generated features from different perspectives: predefined semantics, the diffusion process, and SC-based representations. To integrate knowledge from all discriminators, we propose a mutual learning loss based on the Wasserstein distance, further reinforcing substantial correlations.

Next, for enhancing the correlation on unseen classes, we propose two key insights on the generating stage (Fig.~\ref{fig:motivation}(d)):
(4) \textbf{Diffusion-based Test-time Adaption (DiffTTA)}: Since the generator only encountered seen correlations during the training stage, its generation of unseen features is inaccurate. At the unseen class generating stage, we use a pre-classifier to obtain the pseudo labels of testing samples, and propose a test-time adaptation loss by feature reconstruction for stable adaptation.
(5) \textbf{Diffusion-based Test-time Generation (DiffGen)}: To further enhance the connection between real and generated unseen class features, we utilize the trained diffusion-based generator to trace the diffuse-denoise path of unseen class features with different noise ratios.
Since the perturbed noises varied from fully noise to near-zero noise, the produced features can be seen as fully synthesized ($t=T$), partially synthesized ($0<t<T$), or fully reconstructed ($t=0$).
Surprisingly, we find that partially synthesized features of unseen classes not only mitigate data scarcity further, but also allow us trace every synthesized feature to corresponding real features.
Extensive experiments showcase that ZeroDiff++ achieves the new state-of-the-art on three ZSL benchmarks with various scales of training set.

A preliminary version of this work was presented at a conference, referred to as ZeorDiff~\cite{ye2025zerodiff} (cf Fig.~\ref{fig:banner}(b)).
 \textbf{In this paper, we further extend the robust visual-semantic correlation from seen classes to unseen classes}. We summarize our extensions in four points:
i) We propose DiffTTA with pseudo labels to adapt the generator by feature reconstruction.
ii) We propose DiffGen to partially synthesize unseen class features, allowing us to trace the feature generation process and further mitigate the data-scarce ZSL.
iii) More importantly, although previous work~\cite{xiao2022tackling,wang2022diffusion} also discussed that diffusion-based discriminators can alleviate the overfitting of GAN, this work is the first to theoretically prove it.
iv) We also conduct a fair comparison with the previous diffusion-based ZSL method diffusionZSL~\cite{li2023your} that directly classifies samples by the reconstruction errors of a trained diffusion generator. The experimental results in Sec. ~\ref{sec:comparison_inference} show that this method still exhibits significant classification bias and is not better than traditional generative zero-shot learning methods.

To summarize, our contributions are as follows:
\begin{itemize}
        \item We reveal and quantify the spurious visual-semantic correlation problem, and empirically demonstrate that the problem would be amplified by fewer training samples, and propose a novel generative ZSL framework, ZeroDiff++, that advances more efficient ZSL in scenarios with limited seen class samples.
        
        \item  We propose three train-time insights to strengthen the seen visual-semantic correlation, which are diffusion augmentation, dynamic representation of limited samples, and multi-view discrimination.
        
        \item We also propose two test-time insights for unseen correlation enhancement, which are DiffTTA that reconstructs unseen features to adapt and DiffGen that injects partially real features into fake features with different diffusion ratios.

        \item We theoretically demonstrated for the first time that diffusion-based discriminators can alleviate overfitting of GANs.
        
        \item We introduce a new protocol to evaluate generative ZSL methods under varying data conditions. Experimental results show that ZeroDiff++ outperforms various generative models across different amounts of training samples.
    \end{itemize}

\section{Related Work}

ZSL is a research area focused on class-level generalizability~\cite{xian2018zero}. The primary approaches to ZSL can be categorized into embedding and generative methods. Embedding methods~\cite{akata2015evaluation, yang2016zero, ding2017low, chen2022transzero, ye2023rebalanced} learn a direct mapping from visual to semantic spaces or vice versa. For example, TransZero++~\cite{chen2022transzero++} presents a cross-modal transformer-based architecture to ZSL. However, when training samples are limited, embedding methods often underperform compared to generative methods on smaller datasets\cite{chen2022transzero, ye2023rebalanced}. In contrast, generative ZSL methods use various generative models~\cite{martin2017wasserstein} to synthesize visual features for unseen classes and then train a final classifier for these classes.

This paper distinguishes itself from previous works in five key aspects: 
(1) Recent generative ZSL methods have also identified incorrect visual-semantic correlations~\cite{ye2021disentangling, chen2024causal}, but they lack in-depth quantitative analysis. In contrast, we verify this issue using discriminator scores and demonstrate that the problem is exacerbated by a limited number of training samples.
(2) Some ZSL works~\cite{xian2019fvaegan,narayan2020latent, wang2023bi} combine GANs and VAEs to address the well-known mode collapse issue~\cite{luo2024dyngan}. However, we show that VAEGAN-based methods still experience mode collapse when the training set is reduced. As a solution, we propose integrating the diffusion mechanism~\cite{ho2020denoising} to mitigate this problem and a Wasserstein-distance-based mutual learning approach to distill knowledge across multiple discriminators.
(3) To address the limited discriminative power of predefined semantics, recent works propose dynamically updating these semantics, like DSP~\cite{chen2023evolving} and VADS~\cite{hou2024visual}. Our approach revisits the classical SC learning~\cite{khosla2020supervised} and argues that SC-based representations can serve as a new source for instance-level semantics due to their high intra-class variation~\cite{islam2021broad}.
(4) As far as we know, this work is the first to explore the techniques on test-time among generative ZSL methods. Our proposed DiffTTA and DiffGen could reinforce the unseen visual-semantic correlation, further mitigating the scarce training sample problem.
(5) Recent studies~\cite{clark2024text, li2023your, rombach2022high} have shown that specific large-scale diffusion models, like Stable Diffusion, also possess zero-shot classification abilities. For example, diffusionZSL~\cite{li2023your} reconstructs unseen images by using the semantics of all categories and selects the category with the smallest error to zero-shot classify. However, they do not strictly ensure that unseen classes are excluded from training and rely on large-scale models with huge parameters and extensive training sets. In contrast, our interpretation of diffusion' stays true to its core principle: a generative paradigm that learns data distributions by denoising noised data.
More importantly, our method \textbf{does not} violate the ZSL premise~\cite{xian2018zero}. Our experiment also demonstrated that ZeroDiff++ has better performance in unseen category classification, while diffusionZSL is inferior to traditional ZSL inference and still suffers from severe seen-unseen bias.

%\textbf{Data-efficient Generative Model.}
%Exploring data-efficient generative models is crucial, as the success of generative methods often depends on collecting a vast amount of diverse training samples, which is both costly and challenging~\cite{webster2019detecting, saha2022ganorcon}. In previous works, DiffAugment~\cite{zhao2020differentiable} introduced differentiable augmentation for GANs and successfully trained with only 10\% of the data. AdvAug~\cite{chen2021data} demonstrated that a specific GAN architecture could reduce the amount of required training data using the lottery ticket hypothesis. PatchDiff~\cite{wang2023patch} developed a new conditional score function at the patch level. Denoising Diffusion GAN (DDGAN)~\cite{xiao2022tackling} suggested that combining diffusion models with GANs could reduce overfitting, though without sufficient empirical validation. Generative ZSL methods can also be seen as a form of data-efficient generative models, as they eliminate the need to collect data or train models for unseen classes.

\section{ZeroDiff++}
We first build the notation system of the ZSL problem in Sec.~\ref{sec:notations}. Then, we present the spurious vision-semantic correlation in the existing method and our proposed metric in Sec.~\ref{sec:SpuriousVSC}. To solidify the visual-semantic correlation, we introduce the training stage of our ZeroDiff++ in Sec.~\ref{sec:zerodiff++_training}, including the three train-time insights and the generating stage in Sec.~\ref{sec:zerodiff++_generating} with two test-time insights.

\subsection{Notations}
\label{sec:notations}
In the ZSL setting, there are two disjoint label sets: a seen set \( \mathcal{Y}^{s} \) used for training and an unseen set \( \mathcal{Y}^{u} \) used for testing, where \( \mathcal{Y}^{s} \cap \mathcal{Y}^{u} = \emptyset \). The training dataset is denoted as \( \mathcal{D}^{tr} = \{(\mathbf{x}^{s}, y^{s}, \mathbf{a}^{s}) \mid \mathbf{x}^{s} \in \mathcal{X}^{s}, y^{s} \in \mathcal{Y}^{s}, \mathbf{a}^{s} \in \mathcal{A}^{s} \} \), where \( \mathcal{X}^{s} \), \( \mathcal{A}^{s} \), and \( \mathcal{Y}^{s} \) represent the image, semantic, and label spaces for the seen classes. The goal of ZSL is to use the training dataset \( \mathcal{D}^{tr} \) to create a classifier that can classify unseen images in the testing dataset \( \mathcal{D}^{te} = \mathcal{D}^{u} = \{(\mathbf{x}^{u}, y^{u}, \mathbf{a}^{u}) \mid x^{u} \in \mathcal{X}^{u}, y^{u} \in \mathcal{Y}^{u}, \mathbf{a}^{u} \in \mathcal{A}^{u} \} \), i.e., \( f_{zsl}: \mathcal{X}^{u} \to \mathcal{Y}^{u} \). In the Generalized ZSL (GZSL) task, images from seen classes must also be classified during testing. Therefore, a portion of the seen class samples is reserved for testing, denoted as $\mathcal{D}^{te,s}$. In other words, the testing dataset becomes \( \mathcal{D}^{te} = \mathcal{D}^{te,s} \cup \mathcal{D}^{u} \), and the goal becomes \( f_{gzsl}: \mathcal{X}^{s} \cup \mathcal{X}^{u} \to \mathcal{Y}^{u} \cup \mathcal{Y}^{s} \).

\subsection{Spurious Vision-Semantic Correlation}
\label{sec:SpuriousVSC}

\begin{figure}
\centering
\includegraphics[width=\linewidth]{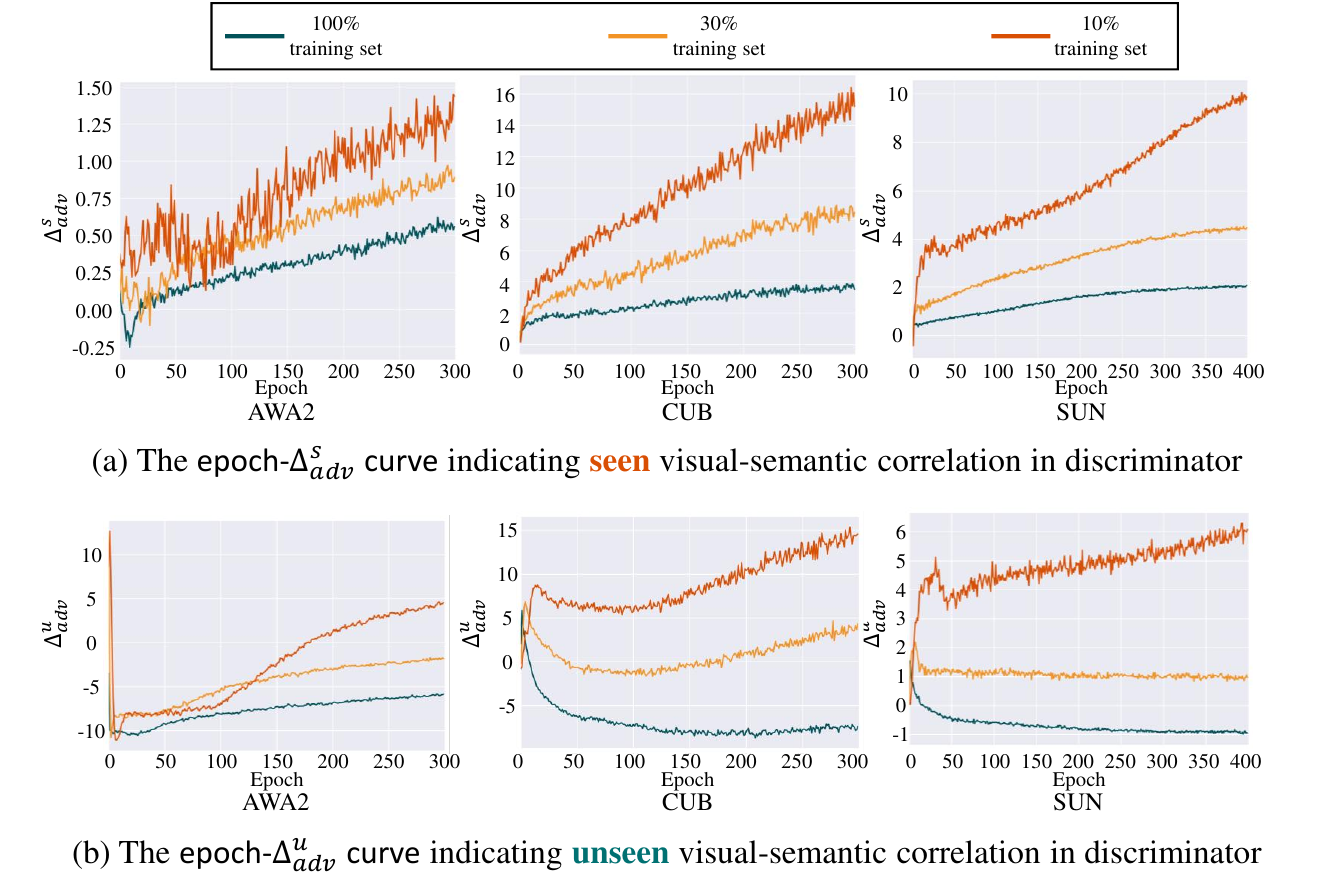}
\caption{The epoch-$\Delta^s_{adv}$ and epoch-$\Delta^u_{adv}$ curves of the classical f-VAEGAN~\cite{xian2018feature}. Larger $\Delta^s_{adv}$ indicates $D_{adv}$ thinks real testing seen examples are more fake, i.e., learns more spurious visual-semantic correlation on seen classes. Similarly, the epoch-$\Delta^u_{adv}$ curve reveals the learned unseen correlation of $D_{adv}$.
}
\label{fig:VSC}
\end{figure}

Given an image \( \mathbf{x}^{s} \) from the training set \( \mathcal{D}^{tr} \), existing GAN-based ZSL works use a feature extractor \( F^*_{ce} \) pre-trained by Cross-Entropy (CE) loss to extract its visual features \( \mathbf{v}^{s}_{0} = F^*_{ce}(\mathbf{x}^{s}) \): $\mathcal{L}_{CE} = - \sum^{|\mathcal{Y}|}_{i=1} \mathbf{y}_{i} \log(\mathbf{\widehat{y}_i})$.
Then, a feature generator $G$ takes class semantics $\mathbf{a}^{s}$ and latent variables $\mathbf{z}$ as inputs to synthesize class-specific sample features \( \tilde{\mathbf{v}}^{s}_{0} = G_{adv}(\mathbf{a}^{s}, \mathbf{z}) \); and a discriminator $D_{adv}$ is used to distinguish real features $\mathbf{v}^{s}_{0}$ from fake features $\tilde{\mathbf{v}}^{s}_{0}$ by predicting the Wasserstein distance $W_{adv}$ between them according to predefined semantics $\mathbf{a}^{s}$.
Specifically, take the baseline f-VAEGAN as an example, its $D_{adv}$ maximizes the following loss $\mathcal{L}_{adv}$:
\begin{align}
    \label{eq:D_adv}
    \mathcal{L}_{adv}  = W_{adv} - \lambda_{gpadv} \mathcal{L}_{gpadv},
\end{align}
\begin{align}
    \label{eq:w_adv}
    W_{adv} & = \mathbb{E}[D_{adv}(\mathbf{v}^{s}_0,\mathbf{a}^{s})] - \mathbb{E}[D_{adv}(\tilde{\mathbf{v}}^{s}_{0},\mathbf{a}^{s})],
\end{align}
\begin{align}
    \label{eq:gpadv}
    \mathcal{L}_{gpadv} = \mathbb{E}[(\|\nabla_{\hat{\mathbf{v}}^{s}_{0}} D_{adv}(\hat{\mathbf{v}}^{s}_{0},\mathbf{a}^{s})\|_{2} - 1)^{2}],
\end{align}
where \( \hat{\mathbf{v}}^{s}_{0} = \alpha \mathbf{v}^{s}_{0} + (1-\alpha) \tilde{\mathbf{v}}^{s}_{0} \) with \( \alpha \sim U(0,1) \), and \( \lambda_{gpadv} \) is a coefficient of the gradient penalty term \( \mathcal{L}_{gpadv} \)~\cite{gulrajani2017improved} that aims to stabilize the training of GANs.

Since a higher critic score means that discriminators consider the input data more `real', we take the critic score difference  between real training features \( \mathbf{v}^{s}_{0} \) and testing seen class features \( \mathbf{v}^{te,s}_{0} \) to indicate whether the learned seen vision-semantic correlation is spurious or substantial, i.e.,
\begin{align}
    \label{eq:delta_adv_seen}
    \Delta^s_{adv}(\mathbf{v}^{s}_{0}, \mathbf{v}^{te,s}_{0}) &= D_{adv}(\mathbf{v}^{s}_{0},\mathbf{a}^{s}) - D_{adv}(\mathbf{v}^{te,s}_{0},\mathbf{a}^{te,s}).
\end{align}
Similarly, we can count the critic score difference  between real training features \( \mathbf{v}^{s}_{0} \) and testing unseen class features \( \mathbf{v}^{u}_{0} \) to indicate the unseen correlation is spurious or substantial, i.e.,
\begin{align}
    \label{eq:delta_adv_unseen}
    \Delta^u_{adv}(\mathbf{v}^{s}_{0}, \mathbf{v}^{te,s}_{0}) &= D_{adv}(\mathbf{v}^{s}_{0},\mathbf{a}^{s}) - D_{adv}(\mathbf{v}^{u}_{0},\mathbf{a}^{u}).
\end{align}

Taking this metric, we display the results of f-VAEGAN in Fig.~\ref{fig:VSC}.
We can find $\Delta^s_{adv}$ and $\Delta^u_{adv}$ continuously increase, and the difference in critic score becomes more significant on smaller training sets, showing that the learned both seen and unseen vision-semantic correlations are spurious; this explains why the performances of the generative ZSL methods could drop when the training set shrinks.

\subsection{Training Stage of ZeroDiff++}
\label{sec:zerodiff++_training}

\begin{figure*}
    \centering
    \includegraphics[width=\linewidth]{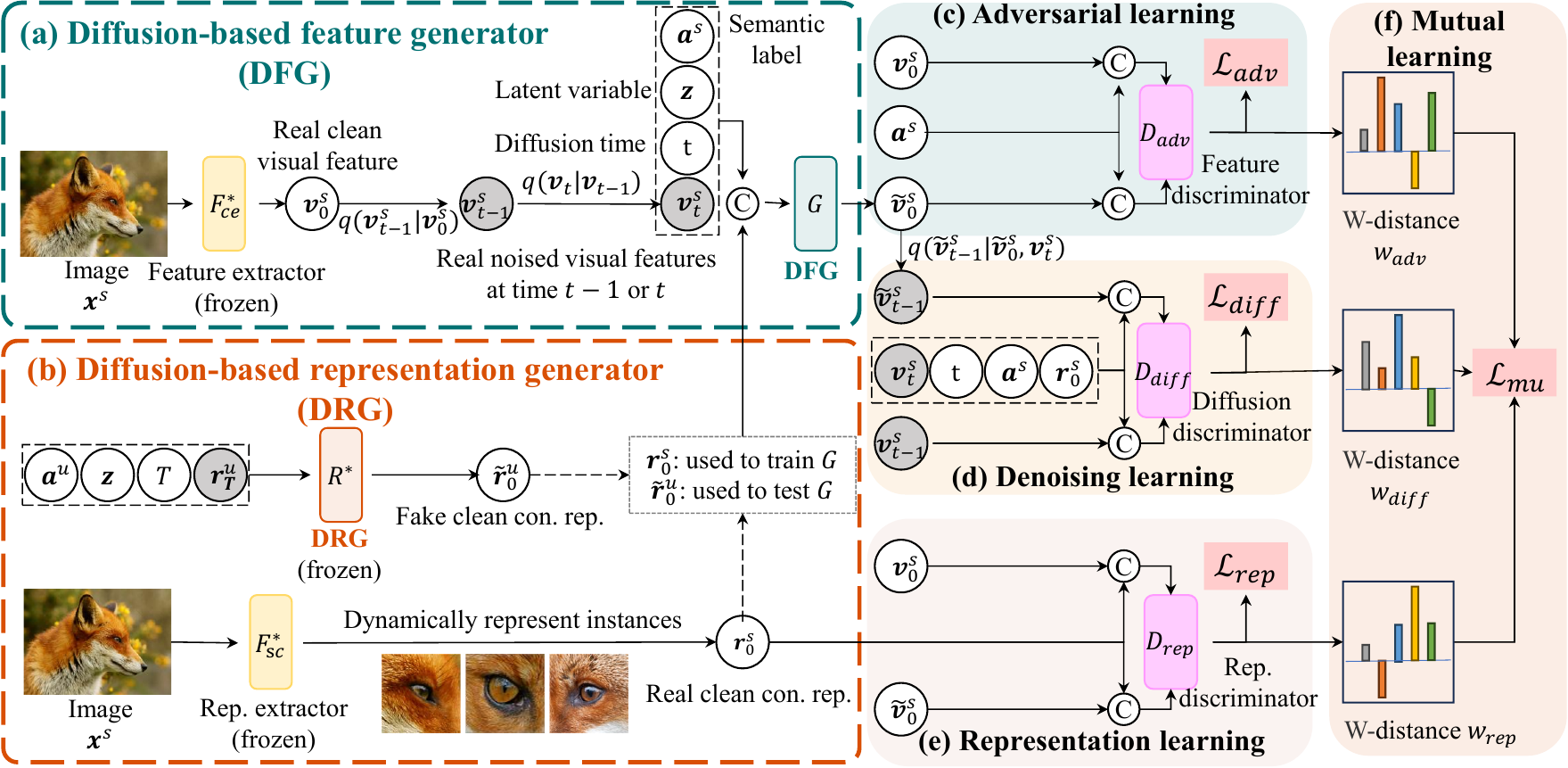}
    \caption{The training stage of our ZeroDiff++. \textcircled{c} represents the concatenation operation. Given frozen extractors \( F^{*}_{ce} \) and \( F^{*}_{sc} \), we take them to extract clean visual features \( \mathbf{v}_{0} \) and contrastive representations \( \mathbf{r}_{0} \). Then, we use the diffusion forward chain (Eq.~\ref{eq:diff_forward_chain}) to obtain real noised visual features \( \mathbf{v}_{t-1} \) and \( \mathbf{v}_{t} \). Next, \(G\) in DFG denoises/generates a fake clean feature \( \tilde{\mathbf{v}}_{0} \), conditioned by the concatenation of the semantic label \( \mathbf{a} \), latent variable \( \mathbf{z} \), diffusion time \( t \), noised feature \( \mathbf{v}_{t} \), and SC-based representation \( \mathbf{r}_{0} \). The fake clean feature is evaluated from three different learning perspectives: adversarial learning (\emph{Does it match predefined semantics?}), denoising learning (\emph{ Does it match diffusion processes?}), and representation learning (\emph{Does it match contrastive representations?}). Finally, we present the mutual learning loss \( \mathcal{L}_{mu} \) to integrate knowledge of all discriminators.}
    \label{fig:zerodiff}
\end{figure*}

We now gradually introduce our proposed components in the training stage that are motivated by our three insights and corresponding targeted problems. These components are integrated into the baseline f-VAEGAN, resulting in the proposed ZeroDiff++.

\subsubsection{\textbf{Key1: Diffusion Augmentation}}
\label{sec:diffAugTrain}

\begin{figure*}
\centering
\includegraphics[width=\linewidth]{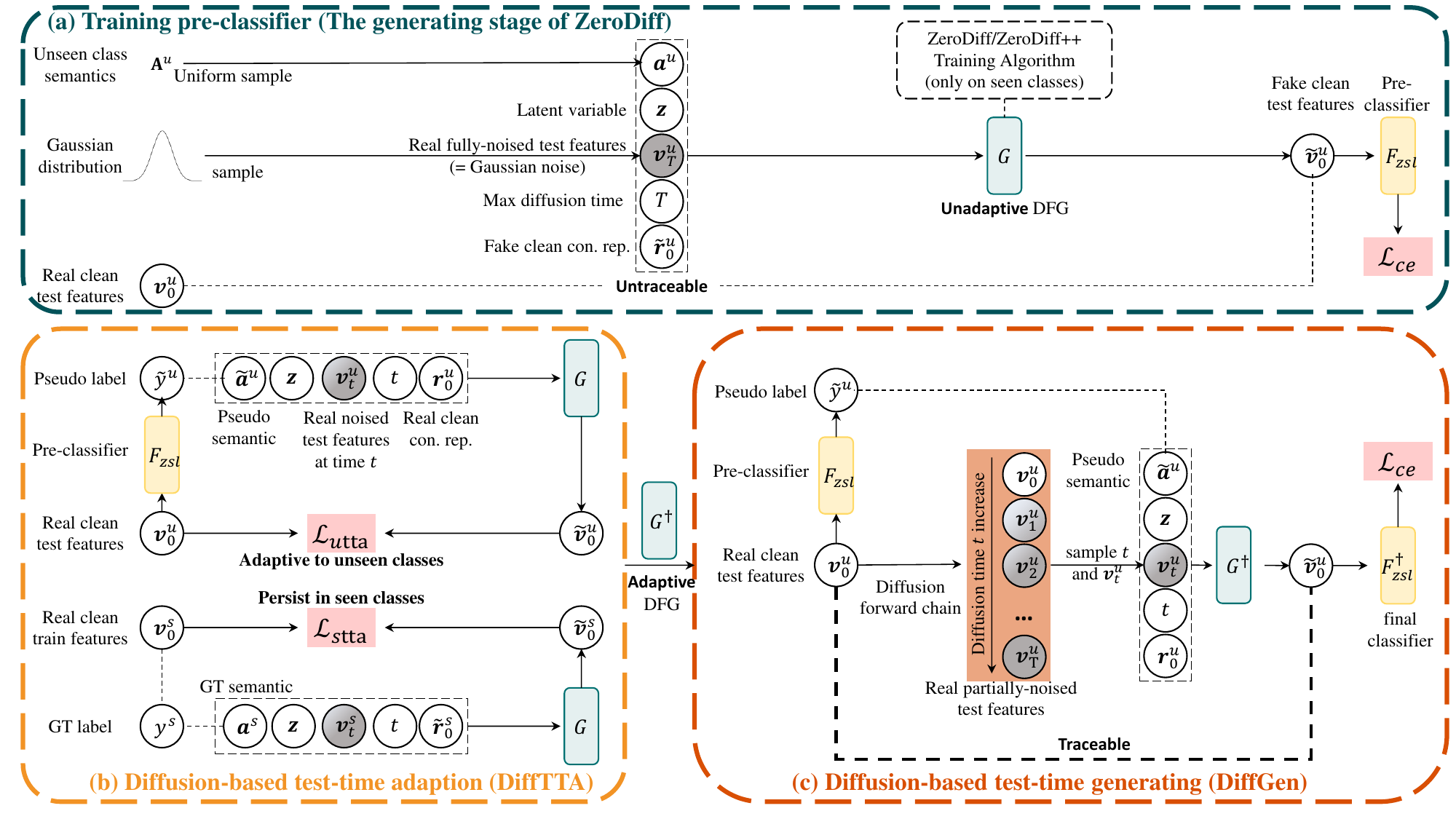}
\caption{The generating stage of our ZeroDiff++. In the generating stage, Zerodiff++ firstly trains a pre-classifier $F_{zsl}$ by unadaptive DFG. Next, it adapts unseen classes by the $\mathcal{L}_{utta}$ loss to minimize the denoising error between real test features $\mathbf{v}^{u}_{0}$ and fake test features $\tilde{\mathbf{v}}^{u}_{0}$ conditioned by the pseudo semantics from pre-classifier. Besides, it also uses the $\mathcal{L}_{stta}$ to prevent forgetting seen classes, leading to the adaptive DFG $G^{\dagger}$. Finally, it samples the partially-noised test features $\mathbf{v}^{u}_{t}$ from the diffusion forward chain, generates the partially-denoised test features $\tilde{\mathbf{v}}^{u}_{t}$, and training the final ZSL classifier $F^{\dagger}_{zsl}$.}
\label{fig:testTime}
\end{figure*}

Following traditional GAN-based methods, we also adapt $D_{adv}$ and the traditional adversarial loss (Eq.~\ref{eq:D_adv}) to determine whether the generated features align with the predefined semantics, as shown in Fig.~\ref{fig:zerodiff}(c).
However, one of the causes of spurious visual-semantic correlation is \textbf{Discriminator Overfitting}, i.e., limited training sets are memorized by $D_{adv}$.
To this end, the first key insight is to leverage the diffusion mechanism to augment the limited training set into infinite noisy data to mitigate the overfitting.
Specifically, we propose our Diffusion-based Feature Generator (DFG) $G$, as shown in Fig.~\ref{fig:zerodiff} (a).

\textbf{Diffusion-based Generating:}
Given an image \( \mathbf{x}^{s} \) from the training set \( \mathcal{D}^{tr} \), we use \( F^*_{ce} \) and  \( F^*_{sc} \) to extract its clean visual features \( \mathbf{v}^{s}_{0} = F^*_{ce}(\mathbf{x}^{s}) \) and clean contrastive representation \( \mathbf{r}^{s}_{0} = F^*_{sc}(\mathbf{x}^{s}) \) (The \( F^*_{sc} \) and \(\mathbf{r}^{s}_{0}\) are introduced in the Sec.~\ref{sec:SCRep}).
Following previous diffusion models~\cite{ho2020denoising, xiao2022tackling}, we apply the diffusion noising process to generate infinite data at various noise levels, from weak to strong, according to the diffusion forward chain:
\begin{equation}
\label{eq:diff_forward_chain}
q(\mathbf{v}^{s}_{1:T}|\mathbf{v}^{s}_0) = \prod_{t\geq1} q(\mathbf{v}^{s}_{t}|\mathbf{v}^{s}_{t-1}),
\end{equation}
\begin{equation}
\label{eq:diff_forward_chain_step}
q(\mathbf{v}^{s}_t|\mathbf{v}^{s}_{t-1}) = \mathcal{N}(\mathbf{v}^{s}_{t}; \sqrt{1-\beta_{t}}\mathbf{v}^{s}_{t-1}, \beta_{t} \bm{I}),
\end{equation}
where $\beta_{t}$ is a pre-defined variance schedule.
When $t = T$, the noised feature $\mathbf{v}^{s}_{t}$ become fully Gaussian noise.
We set the maximum diffusion time as \( T \) and randomly sample a diffusion time \( t \sim U(1,\cdots,T) \). We denote the noised visual features as \( \mathbf{v}^{s}_{t-1} \) and \( \mathbf{v}^{s}_{t} \) at the diffusion times \( t-1 \) and \( t \), respectively. Next, to model the denoising process, we concatenate the noised feature \( \mathbf{v}^{s}_{t} \), diffusion time \( t \), predefined class semantics \( \mathbf{a}^{s} \), latent variable \( \mathbf{z} \), and contrastive representation \( \mathbf{r}_{0} \) as the input of \( G \), i.e. \( \tilde{\mathbf{v}}^{s}_{0} = G(\mathbf{a}^{s}, \mathbf{r}^{s}_{0}, t, \mathbf{v}^{s}_{t}, \mathbf{z}) \). 
After denoising/generating a clean feature \( \tilde{\mathbf{v}}^{s}_{0} \), it is evaluated from three aspects: adversarial learning, denoising learning, and representation learning.

\textbf{Diffusion-based Discriminating:}
To evaluate whether the denoised/generated clean features align with the diffusion processes, we design the diffusion discriminator \( D_{diff} \), as illustrated in Fig.~\ref{fig:zerodiff}(d).
\( D_{diff} \) needs to approximate the true denoising distribution \( q(\mathbf{v}^{s}_{t-1}|\mathbf{v}^{s}_{t}) \). To this end, we posterior-sample the real noised visual feature $\mathbf{v}^{s}_{t-1}$ as well as the fake noised visual feature $\tilde{\mathbf{v}}^{s}_{t-1}$:
\begin{equation}
\label{eq:renoise}
\tilde{\mathbf{v}}^{s}_{t-1} \sim q(\tilde{\mathbf{v}^{s}}_{t-1}|\tilde{\mathbf{v}}^{s}_{0},\mathbf{v}^{s}_{t}) = \mathcal{N}(\tilde{\mathbf{v}}^{s}_{t-1}; \tilde{\mu}_t(\mathbf{v}^{s}_{t},\tilde{\mathbf{v}}^{s}_{0}),\tilde{\beta}_{t} \bm{I}),
\end{equation}
\begin{align}
\label{eq:renoise_mu}
\tilde{\mu}_t(\mathbf{v}^{s}_{t},\mathbf{v}^{s}_{0}) = \frac{\sqrt{\bar{\alpha}_{t-1}}\beta_{t}}{1-\bar{\alpha}_t} \mathbf{v}^{s}_{0} + \frac{\sqrt{\alpha_{t}}(1-\bar{\alpha}_{t-1})}{1-\bar{\alpha}_t} \mathbf{v}^{s}_{t},
\end{align}
where $\tilde{\beta}_{t} = \frac{1- \bar{\alpha}_{t-1}}{1-\bar{\alpha}_t} \beta_{t}$ and $\bar{\alpha}_t = \prod^{t}_{j=1} (1-\beta_{j})$.
Then, we train \( D_{diff} \) to learn the Wasserstein distance between them:
\begin{align}
    \label{eq:D_diff}
    \mathcal{L}_{diff}  = W_{diff} - \lambda_{gpdiff} \mathcal{L}_{gpdiff},
\end{align}
\begin{align}
\label{eq:w_diff}
W_{diff} &= \mathbb{E}[D_{diff}(\mathbf{v}^{s}_{t-1},\mathbf{v}^{s}_{t},\mathbf{r}^{s}_{0},\mathbf{a}^{s},t)] \\
& - \mathbb{E}[D_{diff}(\tilde{\mathbf{v}}^{s}_{t-1},\mathbf{v}^{s}_{t},\mathbf{r}^{s}_{0},\mathbf{a}^{s},t)],
\end{align}
\begin{equation}
    \label{eq:gpdiff}
    \mathcal{L}_{gpdiff} = \mathbb{E}[(\|\nabla_{\hat{\mathbf{v}}^{s}_{t-1}} D_{diff}(\hat{\mathbf{v}}^{s}_{t-1}, \mathbf{v}^{s}_{t}, \mathbf{r}^{s}_{0}, \mathbf{a}^{s}, t)\|_{2} - 1)^{2}],
\end{equation}
where \( \hat{\mathbf{v}}^{s}_{t-1} = \alpha \mathbf{v}^{s}_{t-1} + (1-\alpha) \tilde{\mathbf{v}}^{s}_{t-1} \) with \( \alpha \sim U(0,1) \). Our \( G \) minimizes \( \mathcal{L}_{gpdiff} \), which equates to minimizing the learned divergence per denoising step:
\begin{align}
\label{eq:D_diff_rewrite}
\sum_{t \geq 1} \mathbb{E}[D_{diff}(q(\mathbf{v}^{s}_{t-1}|\mathbf{v}^{s}_{t})) \| p_{G}(q(\tilde{\mathbf{v}}^{s}_{t-1}|\mathbf{v}^{s}_{t}))].
\end{align}

In Sec.~\ref{sec:theory}, we also theoretically proved that the diffusion-based discriminator $D_{diff}$ has better resistance to overfitting than clean-sample discriminator $D_{adv}$ since the overlap mass would increase with diffusion time $t$, remedying disjoint-support failure, and $D_{diff}$ has a tighter upper bound than $D_{adv}$ for generalization error.

\subsubsection{\textbf{Key2: SC-based Representations}}
\label{sec:SCRep}

Another cause of spurious visual-semantic correlation is the use of \textbf{Static Class-level Semantics}.
In other words, existing semantic labels $\mathbf{a}$ are class-level, meaning all instances within a class share the same semantic label.  
This can result in mismatched correlations between instances and their semantics, as each limited sample may only represent a subset of the predefined semantics.  
For example, as shown in Fig.~\ref{fig:banner}, all images in the ``fox'' class are labeled with the semantics ``red'' and ``white,'' even though white foxes are not ``red.''  
Such mismatches further amplify spurious visual-semantic correlations.

To address the static class semantic problem, we revisit the SC loss~\cite{khosla2020supervised} and point out that SC-based representations could be used to represent instance-level semantics.
Previous work~\cite{islam2021broad} showed that SC-based representations have larger inter-class variation than those of CE-based features.
This indicates that SC-based representations mirror the characteristics of every instance within classes.
Our empirical study also verifies this point in Appendix~\ref{app:sc2cs}.

To this end, we use SC-based representations as a new source of semantics.
Except fine-tuning the CE-based extractor \( F_{ce} \), we also fine-tune \( F_{sc} \) with the SC loss:
$\mathcal{L}_{SC} = - \log \frac{\exp(\mathbf{h}^\top \mathbf{h}^{+}/\tau)}{\exp(\mathbf{h}^\top\mathbf{h}^{+}/\tau)+\sum_{K}^{k=1}\exp(\mathbf{h}^\top \mathbf{h}^{-}_{k}/\tau)}$,
where $\mathbf{h}^{+}$, $\mathbf{h}^{-}$, $K$ and $\tau>0$ are the positive example (have the same class label with $\mathbf{h}$), negative example (have a different class label to $\mathbf{h}$), the number of $\mathbf{h}^{-}$ in the batch, a temperature parameter, respectively.
And then. fix it as \( F^{*}_{sc} \) to extract contrastive representations \( \mathbf{r}^{s}_{0} = F^*_{sc}(\mathbf{x}^{s}) \). The extracted $\mathbf{r}^{s}_{0}$ is taken as the input to $G$ for instance-level semantics, as shown in Fig.~\ref{fig:zerodiff}(b).

\textbf{Representation Discriminating} As shown in Fig.~\ref{fig:zerodiff}(e), the representation discriminator \( D_{rep} \) is responsible for distinguishing features via the contrastive representation view.
It operates in a similar manner to \( D_{adv} \), but with the pre-defined semantics \( \mathbf{a} \) replaced by the contrastive representation:
\begin{align}
    \label{eq:D_rep}
    \mathcal{L}_{rep}  = W_{rep} - \lambda_{gprep} \mathcal{L}_{gprep},
\end{align}
\begin{align}
    \label{eq:w_rep}
    W_{rep} = \mathbb{E}_{q(\mathbf{v}^{s}_{0})}[D_{rep}(\mathbf{v}^{s}_{0},\mathbf{r}^{s}_{0})] - \mathbb{E}_{p_{G}(\tilde{\mathbf{v}}^{s}_{0})}[D_{rep}(\tilde{\mathbf{v}}^{s}_{0},\mathbf{r}^{s}_{0})],
\end{align}
\begin{align}
    \label{eq:gprep}
    \mathcal{L}_{gprep} = \mathbb{E}_{\substack{q(\mathbf{v}^{s}_{0}),\\p_{G}(\tilde{\mathbf{v}}^{s}_{0})} }[(\|\nabla_{\hat{\mathbf{v}}^{s}_{0}}D_{rep}(\hat{\mathbf{v}}^{s}_{0},\mathbf{r}^{s}_{0})\|_{2}-1)^{2}].
\end{align}

\textbf{Unseen Representation Generating}
At the testing stage, we cannot get real unseen class representations $r^{u}_{0}$ and feed them to $G$.
Thus, we train another representation generator DRG $R$ to learning the mapping between instance-level SC representations $r_{0}$ from class-level semantic labels $a$.

\subsubsection{\textbf{Key3: Mutual-learned Discriminators}}
\label{sec:MLD}
As shown in Fig.~\ref{fig:zerodiff}(f), since our three discriminators evaluate features in different ways, they have different criteria for judging them.
If we can enable them to learn mutually, we can obtain stronger discriminators, resulting in better guidance for the generator. For example, the objectives of \( D_{adv} \) and \( D_{diff} \) are two very similar but distinct tasks: one separates clean features, and the other separates noised features. Clearly, separating noisy features is a harder task because with more diffusion steps, less information remains. Thus, if we can distill the knowledge from \( D_{adv} \) to \( D_{diff} \), we can enhance the separation ability on noisy features and use the stronger \( D_{diff} \) to improve denoising. In contrast, distilling the knowledge from \( D_{diff} \) to \( D_{adv} \) could prevent \( D_{adv} \) from memorizing training samples.
To this end, we propose the Wasserstein-distance-based distillation loss:
\begin{align}
\label{eq:L_mu}
\mathcal{L}_{mu} & = \kappa_{t}^{\gamma} * (\| W_{diff} - W_{adv} \|_{1}  + \| W_{diff} - W_{rep} \|_{1}) \nonumber \\
& + \| W_{adv} - W_{rep} \|_{1},
\end{align}
where $\kappa_t$ is the Noise-to-Data (N2D) ratio, represented as $1 - \sqrt{\prod^{t}_{j=1} (1-\beta_{j})}$, and $\gamma$ is a smoothing factor.
As \( t \) increases, \( \kappa_t \) also increases, making it harder to distinguish between fake and real noised features, and \( W_{diff} \) provides less guidance for \( W_{rep} \) and \( W_{adv} \). Therefore, we introduce a smoothing factor \( \gamma \geq 0 \) to control the strength of discriminator alignment.

Combining aforementioned training losses, ZeroDiff++ alternately trains $G$ and three discriminators $D_{adv}$, $D_{diff}$, and $D_{rep}$ to optimize the following objective function:
\begin{equation}
\label{eq:optimization}
    \min_{G} \max_{\substack{D_{adv}, D_{diff} , D_{rep}}} (\mathcal{L}_{adv} + \mathcal{L}_{diff} + \mathcal{L}_{rep} - \lambda_{mu} \mathcal{L}_{mu}),
\end{equation}
where $\lambda_{mu}$ is the hyper-parameter related to the mutual learning.

\subsection{Generating Stage of ZeroDiff++}
\label{sec:zerodiff++_generating}
Before we introduce our proposed generating stage, let us recall the traditional generating stage.
As shown in Fig.~\ref{fig:testTime}(a) and (b), the traditional generating stage simply freezes the trained generator to synthesize fake unseen class features.
Taking ZeroDiff~\cite{ye2025zerodiff} as an example, we denote the frozen DFG and DRG as \( G^{*} \) and \( R^{*} \) and synthesize $N_{syn}$ features for each unseen class:
\begin{equation}
\label{eq:zerodiff_gen}
    \tilde{\mathbf{v}}^{u}_{0} = G^{*}(\mathbf{a}^{u}, \tilde{\mathbf{r}}^{u}_{0}, T, \mathbf{v}^{u}_{T}, \mathbf{z}),
\end{equation}
\begin{equation}
    \tilde{\mathbf{r}}^{u}_{0} = R^{*}(\mathbf{a}^{u}, T, \mathbf{r}^{u}_{T}, \mathbf{z}),
\end{equation}
where $\mathbf{r}^{u}_{T}$ and $\mathbf{v}^{u}_{T}$ are fully gaussian noises, the diffusion time only could be set as the maximal $T$, and $\mathbf{a}^{u}$ are uniformly sampled from the unseen semantic space $A^u$.
We can utilize generated fake unseen class features $\tilde{\mathbf{v}}^{u}_{0}$ to train a ZSL classifier $F_{zsl}$ in the feature space, i.e. $F_{zsl}: \mathcal{V}^{u} \rightarrow \mathcal{Y}^{u}$.

Now, let us find the potential problems in the traditional generating stage.
On one hand, $G$ still is a \textbf{unadaptive generator}: although it learns a good visual-semantic correlation on seen classes after the training stage, the learned correlation still has a non-negligible gap with the unseen visual-semantic correlation because $G$ does not adapt to unseen classes.
On the other hand, the traditional generating stage (Eq.~\ref{eq:zerodiff_gen}) is \textbf{fully-noised generating}: the input of $G$ does not contain any real feature information, leading to missing a connection between real and fake visual features.
We propose the following two insights to overcome these two problems.

\subsubsection{\textbf{Key 4: Diffusion-based Test-time Adaptation}}
For generator adaptation, we propose the Diffusion-based Test-time Adaptation (DiffTTA) that aims to shrink the correlation gap at the generating stage, as shown in Fig.~\ref{fig:testTime}(c).
Specifically, given an unlabeled unseen class image $x^u$, we use $F_{CE}$ and $F_{SC}$ to extract its CE-based feature $\mathbf{v}^{u}_{0}$ and SC-based representation $\mathbf{r}^{u}_{0}$.
However, since its corresponding class label $y^{u}$ is missing, we cannot obtain its ground truth semantic $a^{u}$.
As a result, we cannot directly apply the training optimization~\ref{eq:optimization} on unseen class samples to adapt to the unseen class distribution.
To overcome this point, we use a pre-classifier to obtain it a pseudo label $\tilde{y}^{u}$ and semantic $\tilde{\mathbf{a}}^{u}$.
The pre-classifier could be any existing ZSL classifier, e.g., ZeroDiff $F_{zsl}$ or DiffusionZSL~\cite{li2023your}.

However, pseudo labels might amplify the instability of adversarial optimization~\ref{eq:optimization}. Thus, we propose the unseen class TTA loss $\mathcal{L}_{utta}$ based on feature reconstruction:
\begin{align}
    \label{eq:utta}
    \mathcal{L}_{utta} = \| \mathbf{v}^{u}_{0} - G(\tilde{\mathbf{a}}^{u}, \mathbf{r}^{u}_{0}, t, \mathbf{v}^{u}_{t}, \mathbf{z})\|,
\end{align}
Analogously, we also reconstruct seen class samples to prevent forgetting seen classes:
\begin{align}
    \label{eq:stta}
    \mathcal{L}_{stta} = \| \mathbf{v}^{s}_{0} - G(\mathbf{a}^{s}, \mathbf{r}^{s}_{0}, t, \mathbf{v}^{s}_{t}, \mathbf{z})\|.
\end{align}
And we combine these two reconstruction losses as the total adaptation loss:
\begin{align}
    \label{eq:tta}
    \mathcal{L}_{tta} =\mathcal{L}_{utta} + \mathcal{L}_{stta}.
\end{align}

\subsubsection{\textbf{Key 5: Diffusion-based Test-time Generation}}
To enhance the link between fake and real features, we further propose the Diffusion-based Test-time Generation (\textbf{DiffGen}), as shown in Fig.~\ref{fig:testTime}(d).
It further utilizes the diffusion forward chain to connect the original real test features and the generated fake test features in the generating stage.
Formally, DiffGen generates fake unseen class features by:
\begin{equation}
\label{eq:zerodiffpp_gen}
\tilde{\mathbf{v}}^{u}_{0} = G^{\dagger}(\tilde{\mathbf{a}}^{u}, \mathbf{r}^{u}_{0}, t, \mathbf{v}^{u}_{t}, \mathbf{z}).
\end{equation}
Compared to the traditional generating approach (Eq.~\ref{eq:zerodiff_gen}), our DiffGen has two main properties in solidifying visual-semantic correlation in the generating stage:
\begin{enumerate}
    \item \textbf{Traceable feature generation}: Every generated fake feature could be connected to a real test feature. It not only strongly shrinks the gap between real and fake unseen class features, but also provides an evidential generating method that allows us to trace the failure of generating samples.
    \item \textbf{Partial feature imagination}:
    When $t=0$, the DiffGen (Eq.\ref{eq:zerodiffpp_gen}) becomes fully copy the real test sample $\tilde{\mathbf{v}}^{u}_{0}$.
    When $t=T$, the DiffGen (Eq.\ref{eq:zerodiffpp_gen}) degrades to the traditional fully-noised generating (Eq.\ref{eq:zerodiff_gen}). 
    When $0<t<T$, we could adjust the retention ratio of real test samples $\mathbf{v}^{u}_{0}$ among fake test samples $\tilde{\mathbf{v}}^{u}_{0}$ by changing $t$. It means a trade-off between fully copying real test samples and fully fabricating fake test samples. As a result, it reduces the difficulty of unseen generation because ZeroDiff++ only needs to generate a portion of fake samples, resulting in a more solidified unseen class generation.
    \item \textbf{Paired instance semantics}: Since our DiffGen takes every real noised test features $\mathbf{v}^{u}_{t}$ to feed $G$, we could get the corresponding real SC-based representation $\mathbf{r}^{u}_{0}$ to replace the fake $\tilde{\mathbf{r}}^{u}_{0}$ in Eq.~\ref{eq:zerodiff_gen}.
    Our experimental results also verify the importance of the paired $\mathbf{v}^{u}_{t}$ and $\mathbf{r}^{u}_{0}$.
\end{enumerate}
With DiffGen, we train a new ZSL classifier $F^{\dagger}_{zsl}$ by the generated partial imagined unseen class features $\tilde{\mathbf{v}}^{u}_{0}$.

\section{Experiments}
\label{sec:exp}

\subsection{Dataset}
We conduct experiments on three popular ZSL benchmarks: AWA2~\cite{xian2018zero}, CUB~\cite{welinder2010caltech}, and SUN~\cite{patterson2012sun}.
AWA2 consists of 65 animal classes with 85-D attributes and includes 37,322 images.
CUB is a bird dataset containing 11,788 images of 200 bird species.
SUN has 14,340 images of 645 seen and 72 unseen classes of scenes.
We follow the commonly used setting~\cite{xian2018zero} to divide the seen and unseen classes. We report the average per-class Top-1 accuracy of unseen classes for ZSL. For GZSL, we evaluate the Top-1 accuracy on seen classes (\( S \)) and unseen classes (\( U \)), and report their harmonic mean \( H = \frac{2 \times S \times U}{S + U} \).

\subsection {Implementation Details}
\label{sec:implementation}
For visual features, we extract 2,048-D features for all datasets using ResNet-101~\cite{he2016deep} pre-trained with CE on ImageNet-1K~\cite{deng2009imagenet}.
For contrastive representations, we adopt ResNet-101 pre-trained with PaCo~\cite{cui2021parametric} on ImageNet-1K.
For semantic labels, we use 85-D attribute vectors for AWA2, 1,024-D attributes extracted from textual descriptions~\cite{reed2016learning} for CUB, and 102-D attributes for SUN.
We use Adam to optimize all networks with an initial learning rate of 0.0005. For all datasets, \( \lambda_{gpadv} \), \( \lambda_{gpdiff} \), and \( \lambda_{gprep} \) are fixed at 10.
Following DDGAN~\cite{xiao2022tackling}, the number of diffusion steps \( T \) is set to 4, and we use the discretization of the continuous-time extension, known as the Variance Preserving (VP) SDE~\cite{song2020score} to compute \( \beta_t \) in Eq.~\ref{eq:diff_forward_chain}.
All experiments are conducted on Quadro RTX 8000.

% The random seeds are fixed to 9182, 3483 and 4115 for AWA2, CUB and SUN, respectively.

\begin{table}[htbp]
\setlength\tabcolsep{2pt}
	\centering
	\caption{Comparisons with the state-of-the-arts on ZSL. The best and second-best results are marked in \textcolor{red}{\textbf{Red}} and \textcolor{blue}{\textbf{Blue}}, respectively. $\dagger$ denoted the  results using our fine-tune features, while $\ddagger$ using other fine-tune features. The upper group indicates embedding methods, and the lower group is for generative methods.}
	\begin{tabular*}{\linewidth}{@{\extracolsep\fill}lccccc}
 
	\toprule
  
    \multirow{1}{*}{Method}  & \multirow{1}{*}{Venue} & \multirow{1}{*}{Backbone} & \multicolumn{1}{c}{AWA2} & \multicolumn{1}{c}{CUB} & \multicolumn{1}{c}{SUN} \\
    \midrule

    \multirow{1}{*}{APN} & NeurIPS20 & Res101 & \multicolumn{1}{c}{68.4} & \multicolumn{1}{c}{72.0} & \multicolumn{1}{c}{61.6} \\

    % \multirow{1}{*}{ICIS} & ICCV23 & Res101 & \multicolumn{1}{c}{64.6} & \multicolumn{1}{c}{60.6} & \multicolumn{1}{c}{51.8} \\

    \multirow{1}{*}{TransZero++} & TPAMI22 & Res101 & \multicolumn{1}{c}{72.6} & \multicolumn{1}{c}{78.3} & \multicolumn{1}{c}{67.6} \\
         
    \multirow{1}{*}{ReZSL} & TIP23 & Res101 & \multicolumn{1}{c}{70.9} & \multicolumn{1}{c}{80.9} & \multicolumn{1}{c}{63.0} \\

    \multirow{1}{*}{ZSLViT} & CVPR24  & ViT & \multicolumn{1}{c}{70.7} & \multicolumn{1}{c}{78.9} & \multicolumn{1}{c}{68.3} \\
    
    \multirow{1}{*}{PSVMA++} & TPAMI25 & ViT & \multicolumn{1}{c}{79.2} & \multicolumn{1}{c}{78.8} & \multicolumn{1}{c}{74.5} \\
         
    \midrule
         
    f-CLSWGAN$^\dagger$ & CVPR18  & Res101 & \multicolumn{1}{c}{75.9} & \multicolumn{1}{c}{84.5} & \multicolumn{1}{c}{75.5} \\
        
    f-VAEGAN$^\dagger$ & CVPR19  & Res101 & \multicolumn{1}{c}{75.8} & \multicolumn{1}{c}{85.1} & \multicolumn{1}{c}{75.4} \\
        
    \multirow{1}{*}{TFVAEGAN$^\dagger$ } & ECCV20 & Res101 & \multicolumn{1}{c}{72.4} & \multicolumn{1}{c}{85.8} & \multicolumn{1}{c}{74.1} \\
        
    \multirow{1}{*}{SDVAE$^\dagger$} & ICCV21 & Res101 & \multicolumn{1}{c}{69.3} & \multicolumn{1}{c}{85.1} & \multicolumn{1}{c}{77.0} \\
        
    \multirow{1}{*}{CEGAN$^\dagger$} & CVPR21 & Res101 & \multicolumn{1}{c}{72.8} & \multicolumn{1}{c}{84.6} & \multicolumn{1}{c}{74.1} \\
         
    \multirow{1}{*}{DFCAFlow$^\dagger$} & TCSVT23  & Res101 & \multicolumn{1}{c}{74.4} & \multicolumn{1}{c}{83.9} & \multicolumn{1}{c}{77.2} \\
    
    \multirow{1}{*}{DiffusionZSL$^\dagger$}  & ICCV23 & Res101 & \multicolumn{1}{c}{76.4} & \multicolumn{1}{c}{76.7} & \multicolumn{1}{c}{67.5} \\
    
    \multirow{1}{*}{VADS$^\ddagger$} & CVPR24 & ViT & \multicolumn{1}{c}{82.5} & \multicolumn{1}{c}{86.8} & \multicolumn{1}{c}{76.3} \\
         
    \multirow{1}{*}{ZeroDiff$^\dagger$}  & ICLR25 & Res101 & \multicolumn{1}{c}{\textcolor{blue}{\textbf{87.3}}} & \multicolumn{1}{c}{\textcolor{blue}{\textbf{87.5}}} & \multicolumn{1}{c}{\textcolor{blue}{\textbf{77.3}}} \\
        
    \multirow{1}{*}{\textbf{ZeroDiff++$^\dagger$} }  & \textbf{Ours} & Res101 & \multicolumn{1}{c}{\textcolor{red}{\textbf{93.5}}} & \multicolumn{1}{c}{\textcolor{red}{\textbf{87.7}}} & \multicolumn{1}{c}{\textcolor{red}{\textbf{80.2}}} \\
    \bottomrule
        
    \end{tabular*}
	\label{table:OverallComparsionZSL}
\end{table}

\begin{table*}[htbp]
\setlength\tabcolsep{2pt}
	\centering
	\caption{Comparisons with the state-of-the-arts on GZSL. \( U \), \( S \), and \( H \) represent the top-1 accuracy (\%) of unseen classes, seen classes, and their harmonic mean, respectively. The best and second-best results are marked in \textcolor{red}{\textbf{Red}} and \textcolor{blue}{\textbf{Blue}}, respectively. $\dagger$ denoted the  results using our fine-tune features, while $\ddagger$ using other fine-tune features. The upper group indicates embedding methods, and the lower group is for generative methods.
}
	\begin{tabular*}{\textwidth}{@{\extracolsep\fill}l c c ccc ccc ccc}
 
	\toprule
  
    \multirow{2}{*}{Method} & \multirow{2}{*}{Venue} & \multirow{2}{*}{Backbone} & \multicolumn{3}{c}{AWA2} & \multicolumn{3}{c}{CUB} & \multicolumn{3}{c}{SUN} \\
    \cline{4-6}
    \cline{7-9}
    \cline{10-12}

    & & & \multicolumn{1}{c}{U} & \multicolumn{1}{c}{S} & \multicolumn{1}{c}{H} & \multicolumn{1}{c}{U} & \multicolumn{1}{c}{S} & \multicolumn{1}{c}{H}& \multicolumn{1}{c}{U} & \multicolumn{1}{c}{S} & \multicolumn{1}{c}{H} \\

    \midrule
    
    \multirow{1}{*}{APN} & NeurIPS20 & Res101 & \multicolumn{1}{c}{56.5} & \multicolumn{1}{c}{78.0} & \multicolumn{1}{c}{65.5} & \multicolumn{1}{c}{65.3} & \multicolumn{1}{c}{69.3} & \multicolumn{1}{c}{67.2} & \multicolumn{1}{c}{41.9} & \multicolumn{1}{c}{34.0} & \multicolumn{1}{c}{37.6} \\
    
    \multirow{1}{*}{CLIP*} & ICML21 & ViT & \multicolumn{1}{c}{77.6} & \multicolumn{1}{c}{81.6} & \multicolumn{1}{c}{79.6} & \multicolumn{1}{c}{29.6} & \multicolumn{1}{c}{29.8} & \multicolumn{1}{c}{29.7} & \multicolumn{1}{c}{49.5} & \multicolumn{1}{c}{46.2} & \multicolumn{1}{c}{47.8} \\

    \multirow{1}{*}{CoOp*} & IJCV22 & ViT & \multicolumn{1}{c}{69.4} & \multicolumn{1}{c}{81.3} & \multicolumn{1}{c}{74.9} & \multicolumn{1}{c}{18.2} & \multicolumn{1}{c}{22.2} & \multicolumn{1}{c}{20.0} & \multicolumn{1}{c}{49.3} & \multicolumn{1}{c}{49.8} & \multicolumn{1}{c}{49.5} \\

    \multirow{1}{*}{TransZero++} & TPAMI22 & Res101 & \multicolumn{1}{c}{64.6} & \multicolumn{1}{c}{82.7} & \multicolumn{1}{c}{72.5} & \multicolumn{1}{c}{67.5} & \multicolumn{1}{c}{73.6} & \multicolumn{1}{c}{70.4} & \multicolumn{1}{c}{48.6} & \multicolumn{1}{c}{37.8} & \multicolumn{1}{c}{42.5} \\

    % \multirow{1}{*}{ICIS} & ICCV23 & Res101 & \multicolumn{1}{c}{35.6} & \multicolumn{1}{c}{\textcolor{blue}{\textbf{93.3}}} & \multicolumn{1}{c}{51.6} & \multicolumn{1}{c}{45.8} & \multicolumn{1}{c}{73.7} & \multicolumn{1}{c}{56.5} & \multicolumn{1}{c}{45.2} & \multicolumn{1}{c}{25.6} & \multicolumn{1}{c}{32.7} \\
         
    \multirow{1}{*}{ReZSL} & TIP23 & Res101 & \multicolumn{1}{c}{63.8} & \multicolumn{1}{c}{85.6} & \multicolumn{1}{c}{73.1} & \multicolumn{1}{c}{72.8} & \multicolumn{1}{c}{74.8} & \multicolumn{1}{c}{73.8} & \multicolumn{1}{c}{47.4} & \multicolumn{1}{c}{34.8} & \multicolumn{1}{c}{40.1} \\
         
    \multirow{1}{*}{PSVMA} & CVPR23 & ViT & \multicolumn{1}{c}{73.6} & \multicolumn{1}{c}{77.3} & \multicolumn{1}{c}{75.4}& \multicolumn{1}{c}{70.1} & \multicolumn{1}{c}{77.8} & \multicolumn{1}{c}{73.8} & \multicolumn{1}{c}{61.7} & \multicolumn{1}{c}{45.3} & \multicolumn{1}{c}{52.3}\\

    \multirow{1}{*}{TPR*} & NeurIPS24 & ViT & \multicolumn{1}{c}{76.8} & \multicolumn{1}{c}{87.1} & \multicolumn{1}{c}{81.6} & \multicolumn{1}{c}{26.8} & \multicolumn{1}{c}{41.2} & \multicolumn{1}{c}{32.5} & \multicolumn{1}{c}{45.4} & \multicolumn{1}{c}{50.4} & \multicolumn{1}{c}{47.8} \\

    \multirow{1}{*}{ZSLViT} & CVPR24  & ViT & \multicolumn{1}{c}{66.1} & \multicolumn{1}{c}{84.6} & \multicolumn{1}{c}{74.2} & \multicolumn{1}{c}{69.4} & \multicolumn{1}{c}{78.2} & \multicolumn{1}{c}{73.6} & \multicolumn{1}{c}{45.9} & \multicolumn{1}{c}{48.4} & \multicolumn{1}{c}{47.3}\\

    \multirow{1}{*}{EG} & TPAMI24  & ViT & \multicolumn{1}{c}{65.2} & \multicolumn{1}{c}{79.6} & \multicolumn{1}{c}{71.7} & \multicolumn{1}{c}{65.3} & \multicolumn{1}{c}{66.7} & \multicolumn{1}{c}{66.0} & \multicolumn{1}{c}{54.7} & \multicolumn{1}{c}{59.4} & \multicolumn{1}{c}{57.0} \\

    \multirow{1}{*}{PSVMA++} & TPAMI25  & ViT & \multicolumn{1}{c}{74.2} & \multicolumn{1}{c}{86.4} & \multicolumn{1}{c}{79.8} & \multicolumn{1}{c}{71.8} & \multicolumn{1}{c}{77.8} & \multicolumn{1}{c}{74.6} & \multicolumn{1}{c}{61.5} & \multicolumn{1}{c}{49.4} & \multicolumn{1}{c}{54.8} \\
    
    \midrule
         
    f-CLSWGAN$^\dagger$ & CVPR18  & Res101 & \multicolumn{1}{c}{65.1} & \multicolumn{1}{c}{68.9} & \multicolumn{1}{c}{66.9} & \multicolumn{1}{c}{76.4} & \multicolumn{1}{c}{83.3} & \multicolumn{1}{c}{79.7} & \multicolumn{1}{c}{63.8} & \multicolumn{1}{c}{55.7} & \multicolumn{1}{c}{59.5} \\
        
    f-VAEGAN$^\dagger$ & CVPR19  & Res101 & \multicolumn{1}{c}{67.3} & \multicolumn{1}{c}{65.6} & \multicolumn{1}{c}{66.4} & \multicolumn{1}{c}{77.4} & \multicolumn{1}{c}{\textcolor{blue}{\textbf{83.5}}} & \multicolumn{1}{c}{80.3} & \multicolumn{1}{c}{63.6} & \multicolumn{1}{c}{54.1} & \multicolumn{1}{c}{58.4} \\
        
    \multirow{1}{*}{TFVAEGAN$^\dagger$} & ECCV20 & Res101 & \multicolumn{1}{c}{54.2} & \multicolumn{1}{c}{\textcolor{blue}{\textbf{89.6}}} & \multicolumn{1}{c}{67.5} & \multicolumn{1}{c}{79.0} & \multicolumn{1}{c}{83.3} & \multicolumn{1}{c}{81.1} & \multicolumn{1}{c}{51.8} & \multicolumn{1}{c}{53.8} & \multicolumn{1}{c}{52.8} \\
        
    \multirow{1}{*}{SDVAE$^\dagger$} & ICCV21 & Res101 & \multicolumn{1}{c}{57.0} & \multicolumn{1}{c}{72.3} & \multicolumn{1}{c}{63.8} & \multicolumn{1}{c}{\textcolor{blue}{\textbf{81.6}}} & \multicolumn{1}{c}{74.2} & \multicolumn{1}{c}{77.7}& \multicolumn{1}{c}{62.3} & \multicolumn{1}{c}{56.9} & \multicolumn{1}{c}{59.5} \\
        
    \multirow{1}{*}{CEGAN$^\dagger$} & CVPR21 & Res101 & \multicolumn{1}{c}{57.1} & \multicolumn{1}{c}{89.0} & \multicolumn{1}{c}{69.6} & \multicolumn{1}{c}{78.9} & \multicolumn{1}{c}{80.9} & \multicolumn{1}{c}{79.9} & \multicolumn{1}{c}{58.1} & \multicolumn{1}{c}{57.4} & \multicolumn{1}{c}{57.8}\\
         
    \multirow{1}{*}{DFCAFlow$^\dagger$} & TCSVT23  & Res101 & \multicolumn{1}{c}{67.6} & \multicolumn{1}{c}{81.0} & \multicolumn{1}{c}{73.7} & \multicolumn{1}{c}{77.3} & \multicolumn{1}{c}{82.9} & \multicolumn{1}{c}{80.0} & \multicolumn{1}{c}{63.0} & \multicolumn{1}{c}{\textcolor{blue}{\textbf{59.6}}} & \multicolumn{1}{c}{\textcolor{blue}{\textbf{61.2}}} \\
    
    \multirow{1}{*}{DiffusionZSL$^\dagger$ } & ICCV23 & Res101 & \multicolumn{1}{c}{6.4} & \multicolumn{1}{c}{70.7} & \multicolumn{1}{c}{11.8}  & \multicolumn{1}{c}{15.4} & \multicolumn{1}{c}{58.5} & \multicolumn{1}{c}{24.4} & \multicolumn{1}{c}{2.8} & \multicolumn{1}{c}{36.7} & \multicolumn{1}{c}{5.2} \\

    \multirow{1}{*}{DSP$^\ddagger$} & ICML23 & Res101 & \multicolumn{1}{c}{60.0} & \multicolumn{1}{c}{86.0} & \multicolumn{1}{c}{70.7} & \multicolumn{1}{c}{51.4} & \multicolumn{1}{c}{63.8} & \multicolumn{1}{c}{56.9} & \multicolumn{1}{c}{48.3} & \multicolumn{1}{c}{43.0} & \multicolumn{1}{c}{45.5} \\

    \multirow{1}{*}{VADS$^\ddagger$} & CVPR24 & ViT & \multicolumn{1}{c}{\textcolor{blue}{\textbf{75.4}}} & \multicolumn{1}{c}{83.6} & \multicolumn{1}{c}{79.3} & \multicolumn{1}{c}{74.1} & \multicolumn{1}{c}{74.6} & \multicolumn{1}{c}{74.3} & \multicolumn{1}{c}{\textcolor{blue}{\textbf{64.6}}} & \multicolumn{1}{c}{49.0} & \multicolumn{1}{c}{55.7} \\
         
    \multirow{1}{*}{ZeroDiff$^\dagger$ }  & ICLR25 & Res101 & \multicolumn{1}{c}{74.7} & \multicolumn{1}{c}{89.3} & \multicolumn{1}{c}{\textcolor{blue}{\textbf{81.4}}} & \multicolumn{1}{c}{80.0} & \multicolumn{1}{c}{83.2} & \multicolumn{1}{c}{\textcolor{blue}{\textbf{81.6}}} & \multicolumn{1}{c}{63.0} & \multicolumn{1}{c}{56.9} & \multicolumn{1}{c}{59.8} \\
    
    \multirow{1}{*}{\textbf{ZeroDiff++$^\dagger$} }  & \textbf{Ours} & Res101 & \multicolumn{1}{c}{\textcolor{red}{\textbf{90.7}}} & \multicolumn{1}{c}{\textcolor{red}{\textbf{93.9}}} & \multicolumn{1}{c}{\textcolor{red}{\textbf{92.3}}} & \multicolumn{1}{c}{\textcolor{red}{\textbf{84.4}}} & \multicolumn{1}{c}{\textcolor{red}{\textbf{86.4}}} & \multicolumn{1}{c}{\textcolor{red}{\textbf{85.4}}} & \multicolumn{1}{c}{\textcolor{red}{\textbf{73.4}}} & \multicolumn{1}{c}{\textcolor{red}{\textbf{66.2}}} & \multicolumn{1}{c}{\textcolor{red}{\textbf{69.6}}} \\
    
    \bottomrule
    
    \end{tabular*}
	\label{table:OverallComparsionGZSL}
\end{table*}

\subsection{Comparison with State-of-the-art}
We select representative generative ZSL methods across different generative model types. These methods include:
1. \textbf{GAN-based methods}: f-CLSWGAN~\cite{xian2018feature} and CEGAN~\cite{han2021contrastive};
2. \textbf{VAE-based methods}: SDVAE~\cite{chen2021semantics};
3. \textbf{VAEGAN-based methods}: f-VAEGAN~\cite{xian2019fvaegan}, TFVAEGAN~\cite{narayan2020latent}, DSP~\cite{chen2023evolving}, DML~\cite{zhang2024bridging} and VADS~\cite{hou2024visual};
4. \textbf{Flow-based method}: DFCAFlow~\cite{su2023dual}.
5. \textbf{Diffusion-based method}: DiffusionZSL~\cite{li2023your}.
We also provide comparisons with embedding-based SOTAs: APN~\cite{xu2020attribute}, TransZero++~\cite{chen2022transzero++}, ReZSL~\cite{ye2023rebalanced}, PSVMA~\cite{liu2023progressive}, ZSLViT~\cite{chen2024progressive} and PSVMA++\cite{liu2024psvma+}.
Besides, large-scale vision-language models, i.e., CLIP~\cite{radford2021learning}, CoOp~\cite{zhou2022learning},  and TPR~\cite{chen2024tpr}, have shown significant potential for task-level ZSL ability that allows them to perform well on many downstream tasks, also including zero-shot classification.
However, they do not follow the strict standard class splitting, whose traditional zero-shot methods used and unseen class images are highly possibly included in their pre-training dataset.
Nonetheless, we still include them for better reference.

The ZSL results are presented in Table~\ref{table:OverallComparsionZSL} and GZSL results are presented in Table~\ref{table:OverallComparsionGZSL}.
Our method ZeroDiff++ exhibits consistent and substantial improvements over both embedding-based and generative baselines across three datasets.
We could conclude three points:

\begin{enumerate}
    \item 
    ZeroDiff++ consistently outperforms prior embedding- and generative-based SOTAs. In ZSL setting, ZeroDiff++ achieves 93.5\% / 87.7\% / 80.2\% on AWA2 / CUB / SUN, and in GZSL the harmonic means \(H\) are 92.3\% / 85.4\% / 70.5\%; these correspond to absolute improvements over ZeroDiff of roughly +9.9\%, +3.8\% and +10.7\% in \(H\). The results indicate that ZeroDiff++ not only raises unseen-class accuracy but also substantially improves the seen/unseen balance in the GZSL setting.
    
    \item Experimental results also expose how ZeroDiff++ mitigates the common weakness of many generative approaches.
    Several generators exhibit severe seen/unseen imbalance (high \(S\) but low \(U\)).
    For example, TFVAEGAN on AWA2 reports \(S=89.6\%\), \(U=54.2\%\) \(H=67.5\%\).
    Prior diffusion-based method DiffusionZSL produces extremely low \(U\) and \(H\). By contrast, ZeroDiff++ yields both high \(U\) and high \(S\) (AWA2: \(U=90.7\%\), \(S=93.9\%\), \(H=92.3\%\));
    Results demonstrate that our diffusion-based training and testing modules improve the seen and unseen visual-semantic correlation.

    \item Notably, our ZeroDiff++ also significantly outperforms the large-scale vision-language based methods (e.g., CLIP, CoOp, and TPR).
    CLIP/CoOp/TPR obtain reasonable \(H\) on coarse AWA2 but fail on fine-grained CUB (CLIP \(H\): 79.6\% / 29.7\% / 47.8\% on AWA2 / CUB / SUN), whereas ZeroDiff++ yields +12.7\%, +55.7\% and +22.7\% higher in \(H\) on the same benchmarks.
    The results demonstrate that although large pre-trained vision-language models enojoy large parameters and pre-training data, task-specific ZSL methods still remain highly competitive.

\end{enumerate}

\subsection{Limited Training Data}
\label{sec:dezsl}

\begin{table*}[htbp]
\setlength\tabcolsep{2pt}
	\centering
	\caption{Comparison on limited training data. We evaluate generative methods with 30\% and 10\% training samples. \( T1 \) represents the top-1 accuracy (\%) of unseen classes in ZSL. In GZSL, \( U \), \( S \), and \( H \) represent the top-1 accuracy (\%) of unseen classes, seen classes, and their harmonic mean, respectively. The best and second-best results are marked in \textcolor{red}{\textbf{Red}} and \textcolor{blue}{\textbf{Blue}}, respectively.
}
	\begin{tabular*}{\textwidth}{@{\extracolsep\fill}l cccc cccc cccc}
 
		\toprule
		\multirow{3}{*}{Method} & \multicolumn{4}{c}{AWA2} & \multicolumn{4}{c}{CUB}  & \multicolumn{4}{c}{SUN} \\
        \cmidrule{2-5}
        \cmidrule{6-9}
        \cmidrule{10-13}

         & \multicolumn{2}{c}{$30\% \mathcal{D}^{tr}$} & \multicolumn{2}{c}{$10\% \mathcal{D}^{tr}$} & \multicolumn{2}{c}{$30\% \mathcal{D}^{tr}$} & \multicolumn{2}{c}{$10\% \mathcal{D}^{tr}$} & \multicolumn{2}{c}{$30\% \mathcal{D}^{tr}$} & \multicolumn{2}{c}{$10\% \mathcal{D}^{tr}$} \\
        \cmidrule{2-3}
        \cmidrule{4-5}
        \cmidrule{6-7}
        \cmidrule{8-9}
        \cmidrule{10-11}
        \cmidrule{12-13}
         
         & \multicolumn{1}{c}{T1} & \multicolumn{1}{c}{H} & \multicolumn{1}{c}{T1} & \multicolumn{1}{c}{H} & \multicolumn{1}{c}{T1} & \multicolumn{1}{c}{H} & \multicolumn{1}{c}{T1} & \multicolumn{1}{c}{H} & \multicolumn{1}{c}{T1} & \multicolumn{1}{c}{H} & \multicolumn{1}{c}{T1} & \multicolumn{1}{c}{H} \\
         \midrule
         f-CLSWGAN~\cite{xian2018feature} & \multicolumn{1}{c}{68.9} & \multicolumn{1}{c}{57.8} & \multicolumn{1}{c}{54.0} & \multicolumn{1}{c}{35.7} & \multicolumn{1}{c}{82.1} & \multicolumn{1}{c}{75.3} & \multicolumn{1}{c}{75.1} & \multicolumn{1}{c}{66.8} & \multicolumn{1}{c}{73.4} & \multicolumn{1}{c}{51.5} & \multicolumn{1}{c}{66.9} & \multicolumn{1}{c}{29.3} \\
         
        f-VAEGAN~\cite{xian2019fvaegan} & \multicolumn{1}{c}{81.2} & \multicolumn{1}{c}{64.9} & \multicolumn{1}{c}{73.1} & \multicolumn{1}{c}{54.4} & \multicolumn{1}{c}{84.5} & \multicolumn{1}{c}{78.3} & \multicolumn{1}{c}{81.5} & \multicolumn{1}{c}{75.0} & \multicolumn{1}{c}{72.0} & \multicolumn{1}{c}{49.0} & \multicolumn{1}{c}{58.4} & \multicolumn{1}{c}{27.8} \\
        
        \multirow{1}{*}{CEGAN~\cite{han2021contrastive}} & 
        \multicolumn{1}{c}{72.2} & \multicolumn{1}{c}{70.4} & \multicolumn{1}{c}{69.0} & \multicolumn{1}{c}{66.3} & \multicolumn{1}{c}{83.6} & \multicolumn{1}{c}{77.6} & \multicolumn{1}{c}{81.3} & \multicolumn{1}{c}{74.8} & \multicolumn{1}{c}{65.9} & \multicolumn{1}{c}{45.6} & \multicolumn{1}{c}{-} & \multicolumn{1}{c}{-} \\
        
         \multirow{1}{*}{DFCAFlow~\cite{su2023dual}} & \multicolumn{1}{c}{74.5} & \multicolumn{1}{c}{72.6} & \multicolumn{1}{c}{77.9} & \multicolumn{1}{c}{70.7} & \multicolumn{1}{c}{82.7} & \multicolumn{1}{c}{77.2} & \multicolumn{1}{c}{80.3} & \multicolumn{1}{c}{74.1} & \multicolumn{1}{c}{74.2} & \multicolumn{1}{c}{\textcolor{blue}{\textbf{55.8}}} & \multicolumn{1}{c}{70.2} & \multicolumn{1}{c}{31.3} \\

        \multirow{1}{*}{DiffusionZSL~\cite{li2023your}} & \multicolumn{1}{c}{75.1} & \multicolumn{1}{c}{11.0} & \multicolumn{1}{c}{70.5} & \multicolumn{1}{c}{10.7} & \multicolumn{1}{c}{76.1} & \multicolumn{1}{c}{24.1} & \multicolumn{1}{c}{73.0} & \multicolumn{1}{c}{23.2} & \multicolumn{1}{c}{64.5} & \multicolumn{1}{c}{4.3} & \multicolumn{1}{c}{19.7} & \multicolumn{1}{c}{0.3} \\
        
        \multirow{1}{*}{ZeroDiff~\cite{ye2025zerodiff} } & \multicolumn{1}{c}{\textcolor{blue}{\textbf{84.9}}} & \multicolumn{1}{c}{\textcolor{blue}{\textbf{80.2}}} & \multicolumn{1}{c}{\textcolor{blue}{\textbf{83.3}}} & \multicolumn{1}{c}{\textcolor{blue}{\textbf{77.0}}} & \multicolumn{1}{c}{\textcolor{blue}{\textbf{85.5}}} & \multicolumn{1}{c}{\textcolor{blue}{\textbf{78.7}}} & \multicolumn{1}{c}{\textcolor{blue}{\textbf{82.9}}} & \multicolumn{1}{c}{\textcolor{blue}{\textbf{76.1}}} & \multicolumn{1}{c}{\textcolor{blue}{\textbf{75.4}}} & \multicolumn{1}{c}{51.3} & \multicolumn{1}{c}{\textcolor{blue}{\textbf{68.1}}} & \multicolumn{1}{c}{\textcolor{blue}{\textbf{33.3}}} \\

        \multirow{1}{*}{\textbf{ZeroDiff++ (Ours)} } & \multicolumn{1}{c}{\textcolor{red}{\textbf{92.8}}} & \multicolumn{1}{c}{\textcolor{red}{\textbf{91.9}}} & \multicolumn{1}{c}{\textcolor{red}{\textbf{93.8}}} & \multicolumn{1}{c}{\textcolor{red}{\textbf{92.0}}} & \multicolumn{1}{c}{\textcolor{red}{\textbf{87.4}}} & \multicolumn{1}{c}{\textcolor{red}{\textbf{82.8}}} & \multicolumn{1}{c}{\textcolor{red}{\textbf{85.2}}} & \multicolumn{1}{c}{\textcolor{red}{\textbf{77.3}}} & \multicolumn{1}{c}{\textcolor{red}{\textbf{77.9}}} & \multicolumn{1}{c}{\textcolor{red}{\textbf{59.2}}} & \multicolumn{1}{c}{\textcolor{red}{\textbf{75.0}}} & \multicolumn{1}{c}{\textcolor{red}{\textbf{41.0}}} \\
        
        \bottomrule
	\end{tabular*}
	\label{table:DEZSL}
\end{table*}

To investigate the performance of generative ZSL approaches under limited training data, we randomly keep training samples of each seen class at different ratios. As the remaining ratio decreases, ZSL becomes more challenging since over-fitting to the training samples becomes much easier. For fairness, we decrease the number of synthesized unseen features proportionally to avoid class imbalance in GZSL. The results are reported in Table~\ref{table:DEZSL}.

The following observations can be made:
\begin{enumerate}
    \item GAN-based approaches, f-CLSWGAN and CEGAN, deteriorate severely when only limited training data are available. For example, f-CLSWGAN drops from T1=68.9\% (30\% $\mathcal{D}^{tr}$) to T1=54.0\% (10\% $\mathcal{D}^{tr}$) on AWA2, and its GZSL \(H\) falls from 57.8\% to 35.7\%, indicating that vanilla GANs are fragile under data scarcity and prone to overfitting or mode collapse. CEGAN is slightly more stable but still shows clear degradation (AWA2 \(H\): $70.4\% \rightarrow 66.3\%$ when moving from $30\% \mathcal{D}^{tr} \rightarrow 10\% \mathcal{D}^{tr}$).
    \item Combining VAE and GAN (f-VAEGAN) alleviates some GAN instability but does not fully solve limited-data failure modes: on SUN with 30\% \(\mathcal{D}^{tr}\) f-VAEGAN attains T1=72.0\% while the simpler f-CLSWGAN reaches T1=73.4\%, and f-VAEGAN’s \(H\) also degrades sharply at 10\% (49.0\% → 27.8\% on SUN). These results show that when the training set is limited, the mode collapse problem still remains, which sometimes performs even worse than the methods that only use GAN.
    \item DiffusionZSL reports reasonably high ZSL T1 but catastrophically low GZSL \(H\) (e.g., AWA2 30\%: T1=75.1\% but \(H=11.0\%\); CUB 30\%: T1=76.1\% but \(H=24.1\%\)). The results consistently show that DiffusionZSL suffers from several class bias problems.
    \item With the introduction of the diffusion mechanism and instance-level representations,  ZeroDiff achieves impressive performances in most cases (ZeroDiff++ 30\% AWA2 \(H=80.2\%\)), and ZeroDiff++ further improves robustness (ZeroDiff++ 30\% AWA2 \(H=91.9\%\)), demonstrating that the diffusion-based adaptation and generation improve sample efficiency and maintain a strong seen/unseen balance under extreme scarcity.
\end{enumerate}

\subsection{Comparison on ZSL Inference methods}
\label{sec:comparison_inference}

\begin{table*}[htbp]
	\centering
	\caption{Comparsion on different inference methods: traditional FNGen, recent DiffusionZSL~\cite{li2023your} and our DiffGen.}
    \begin{tabular*}{\textwidth}{@{\extracolsep\fill}lcccccccccccccc}
 
		\toprule
        
		\multirow{3}{*}{ID} & \multirow{3}{*}{DiffTTA} & \multirow{3}{*}{Inference Method} & \multicolumn{3}{{@{}c@{}}}{ZSL} & \multicolumn{9}{c}{GZSL} \\
        \cmidrule{4-6}
        \cmidrule{7-15}
        
		   &  &  & \multicolumn{1}{c}{AWA2} & \multicolumn{1}{c}{CUB} & \multicolumn{1}{c}{SUN} & \multicolumn{3}{{@{}c@{}}}{AWA2} & \multicolumn{3}{c}{CUB}  & \multicolumn{3}{c}{SUN} \\
        \cmidrule{4-4}
        \cmidrule{5-5}
        \cmidrule{6-6}
        \cmidrule{7-9}
        \cmidrule{10-12}
        \cmidrule{13-15}
        
        & & & \multicolumn{1}{c}{T1} & \multicolumn{1}{c}{T1} & \multicolumn{1}{c}{T1} & \multicolumn{1}{c}{U} & \multicolumn{1}{c}{S} & \multicolumn{1}{c}{H} & \multicolumn{1}{c}{U}& \multicolumn{1}{c}{S} & \multicolumn{1}{c}{H} & \multicolumn{1}{c}{U} & \multicolumn{1}{c}{S} & \multicolumn{1}{c}{H} \\
        \midrule
        
        a & \multirow{3}{*}{$\times$}  & \multicolumn{1}{c}{FNGen} & \multicolumn{1}{c}{85.5} & \multicolumn{1}{c}{86.0} & \multicolumn{1}{c}{77.6}  & \multicolumn{1}{c}{70.5} & \multicolumn{1}{c}{77.5} & \multicolumn{1}{c}{73.8} & \multicolumn{1}{c}{77.6}   & \multicolumn{1}{c}{83.9} & \multicolumn{1}{c}{80.6} & \multicolumn{1}{c}{62.0} & \multicolumn{1}{c}{58.1}  & \multicolumn{1}{c}{60.0} \\

        b & & \multicolumn{1}{c}{DiffusionZSL} & \multicolumn{1}{c}{76.9} & \multicolumn{1}{c}{77.2} & \multicolumn{1}{c}{66.4} & \multicolumn{1}{c}{6.5} & \multicolumn{1}{c}{70.0} & \multicolumn{1}{c}{11.9} & \multicolumn{1}{c}{15.3} & \multicolumn{1}{c}{58.5} & \multicolumn{1}{c}{24.2} & \multicolumn{1}{c}{2.6} & \multicolumn{1}{c}{38.3} & \multicolumn{1}{c}{4.9} \\

        c & & \multicolumn{1}{c}{DiffGen} & \multicolumn{1}{c}{92.9} & \multicolumn{1}{c}{87.0} & \multicolumn{1}{c}{80.0} & \multicolumn{1}{c}{86.4} & \multicolumn{1}{c}{89.7} & \multicolumn{1}{c}{88.0} & \multicolumn{1}{c}{84.3}  & \multicolumn{1}{c}{85.7} & \multicolumn{1}{c}{85.0} & \multicolumn{1}{c}{77.0} & \multicolumn{1}{c}{63.4} & \multicolumn{1}{c}{69.5} \\

        \midrule

        d & \multirow{3}{*}{$\checkmark$} & \multicolumn{1}{c}{FNGen} & \multicolumn{1}{c}{81.1}  & \multicolumn{1}{c}{85.4} & \multicolumn{1}{c}{77.4} & \multicolumn{1}{c}{66.1}  & \multicolumn{1}{c}{82.2} & \multicolumn{1}{c}{73.3} & \multicolumn{1}{c}{77.4} & \multicolumn{1}{c}{85.9} & \multicolumn{1}{c}{81.4}  & \multicolumn{1}{c}{63.1} & \multicolumn{1}{c}{60.5} & \multicolumn{1}{c}{61.8} \\

        e & & \multicolumn{1}{c}{DiffusionZSL} & \multicolumn{1}{c}{80.3} & \multicolumn{1}{c}{78.6} & \multicolumn{1}{c}{69.3} & \multicolumn{1}{c}{7.9} & \multicolumn{1}{c}{70.9} & \multicolumn{1}{c}{14.3} & \multicolumn{1}{c}{16.1} & \multicolumn{1}{c}{57.8} & \multicolumn{1}{c}{25.2} & \multicolumn{1}{c}{2.7} & \multicolumn{1}{c}{37.9} & \multicolumn{1}{c}{5.0} \\

        f & & \multicolumn{1}{c}{DiffGen} & \multicolumn{1}{c}{93.5} & \multicolumn{1}{c}{87.6} & \multicolumn{1}{c}{80.5} & \multicolumn{1}{c}{90.7} & \multicolumn{1}{c}{93.9}& \multicolumn{1}{c}{92.3} & \multicolumn{1}{c}{85.5}  & \multicolumn{1}{c}{86.1} & \multicolumn{1}{c}{85.8} & \multicolumn{1}{c}{78.6} & \multicolumn{1}{c}{64.3} & \multicolumn{1}{c}{70.8} \\

        \bottomrule
	\end{tabular*}
	\label{table:comparsion_inference}
\end{table*}

We conduct an ablation study comparing three inference schemes: (1) the conventional fully-noised feature generation (FNGen), (2) the reconstruction-error classifier approach used by DiffusionZSL, and (3) our DiffGen (partial, traceable generation), as shown in Table~\ref{table:comparsion_inference}. The table shows two consistent patterns:
\begin{enumerate}
    \item  DiffGen strongly outperforms both baselines on ZSL and, crucially, GZSL. For example, when we do not use DiffTTA, DiffGen yields ZSL T1 92.9\% / 87.0\% / 80.0\% (AWA2 / CUB / SUN) vs. FNGen’s 85.5\% / 86.0\% / 77.6\% and DiffusionZSL’s 76.9\% / 77.2\% / 66.4\%; in GZSL (AWA2) DiffGen achieves U/S/H = 86.4\% / 89.7\% / 88.0\%, substantially higher than FNGen (70.5\% / 77.5\% / 73.8\%) and DiffusionZSL (6.5\% / 70.0\% / 11.9\%).
    \item DiffTTA is most beneficial when combined with DiffGen. With DiffTTA enabled, DiffGen improves further (AWA2 GZSL H from 88.0\% $\rightarrow$ 92.3\%, U from 86.4\% $\rightarrow$ 90.7\%), whereas DiffTTA alone does not rescue FNGen or fully fix DiffusionZSL’s severe seen/unseen bias. This indicates that the combination of an adaptive generator (DiffTTA) plus partially real generation (DiffGen) both raises unseen accuracy and restores seen/unseen balance. 
\end{enumerate}

\subsection{Ablation Study}
\label{sec:ablation}
We conduct the component effectiveness study in Sec.~\ref{sec:com} and hyper-parameter sensitivity in Sec.~\ref{sec:hyperpara}. We also provide a detailed ablation on our DiffGen in Sec. ~\ref{sec:ablation_diffgen}.

\begin{table*}[htbp]
	\centering
	\caption{Ablation study on five insightful components of our ZeroDiff++ and the modality of the final ZSL classifier. \( \checkmark \) and \( \times \) denote yes and no. V, C, and S denote CE-based space, SC-based space, and class-level semantic space.}
    \begin{tabular*}{\textwidth}{@{\extracolsep\fill}lcccccccccccc}
 
		\toprule
		\multirow{2}{*}{ID} & \multirow{2}{*}{Classifier} & \multicolumn{5}{{@{}c@{}}}{Component} & \multicolumn{2}{{@{}c@{}}}{AWA2} & \multicolumn{2}{c}{CUB}  & \multicolumn{2}{c}{SUN} \\
        \cmidrule{3-7}
        \cmidrule{8-9}
        \cmidrule{10-11}
        \cmidrule{12-13}
        
        & & \multicolumn{1}{c}{DiffAug} & \multicolumn{1}{c}{SC} & \multicolumn{1}{c}{$\mathcal{L}_{mu}$} & \multicolumn{1}{c}{DiffTTA} & \multicolumn{1}{c}{DiffGen} &  \multicolumn{1}{c}{T1} & \multicolumn{1}{c}{H} & \multicolumn{1}{c}{T1} & \multicolumn{1}{c}{H} & \multicolumn{1}{c}{T1} & \multicolumn{1}{c}{H} \\
        \midrule

        a & \multicolumn{1}{c}{V} & \multicolumn{1}{c}{$\times$} & \multicolumn{1}{c}{$\times$} & \multicolumn{1}{c}{$\times$} & \multicolumn{1}{c}{$\times$} & \multicolumn{1}{c}{$\times$} & \multicolumn{1}{c}{76.0} & \multicolumn{1}{c}{69.5} & \multicolumn{1}{c}{83.1} & \multicolumn{1}{c}{79.7} & \multicolumn{1}{c}{74.3} & \multicolumn{1}{c}{56.9} \\

        b & \multicolumn{1}{c}{V} & \multicolumn{1}{c}{$\checkmark$} & \multicolumn{1}{c}{$\times$} & \multicolumn{1}{c}{$\times$} & \multicolumn{1}{c}{$\times$} & \multicolumn{1}{c}{$\times$} & \multicolumn{1}{c}{80.0} & \multicolumn{1}{c}{72.5} & \multicolumn{1}{c}{83.3} & \multicolumn{1}{c}{79.8} & \multicolumn{1}{c}{75.4} & \multicolumn{1}{c}{59.5} \\

        c & \multicolumn{1}{c}{V} & \multicolumn{1}{c}{$\checkmark$} & \multicolumn{1}{c}{$\checkmark$} & \multicolumn{1}{c}{$\times$} & \multicolumn{1}{c}{$\times$} & \multicolumn{1}{c}{$\times$} & \multicolumn{1}{c}{84.4} & \multicolumn{1}{c}{76.3} & \multicolumn{1}{c}{83.6} & \multicolumn{1}{c}{80.1} & \multicolumn{1}{c}{76.1} & \multicolumn{1}{c}{59.8} \\
        
         d & \multicolumn{1}{c}{V} & \multicolumn{1}{c}{$\checkmark$} & \multicolumn{1}{c}{$\times$} & \multicolumn{1}{c}{$\checkmark$} & \multicolumn{1}{c}{$\times$} & \multicolumn{1}{c}{$\times$} & \multicolumn{1}{c}{81.3} & \multicolumn{1}{c}{73.3} & \multicolumn{1}{c}{84.1} & \multicolumn{1}{c}{80.2} & \multicolumn{1}{c}{76.2} & \multicolumn{1}{c}{59.9} \\
        
        f & \multicolumn{1}{c}{V} & \multicolumn{1}{c}{$\checkmark$} & \multicolumn{1}{c}{$\checkmark$} & \multicolumn{1}{c}{$\checkmark$} & \multicolumn{1}{c}{$\times$} & \multicolumn{1}{c}{$\times$} & \multicolumn{1}{c}{86.4} & \multicolumn{1}{c}{76.9} & \multicolumn{1}{c}{84.3} & \multicolumn{1}{c}{80.2} & \multicolumn{1}{c}{76.4} & \multicolumn{1}{c}{60.0} \\

        g & \multicolumn{1}{c}{V} & \multicolumn{1}{c}{$\checkmark$} & \multicolumn{1}{c}{$\checkmark$} & \multicolumn{1}{c}{$\checkmark$} & \multicolumn{1}{c}{$\checkmark$} & \multicolumn{1}{c}{$\times$} & \multicolumn{1}{c}{86.5} & \multicolumn{1}{c}{78.7} & \multicolumn{1}{c}{84.8} & \multicolumn{1}{c}{80.7} & \multicolumn{1}{c}{76.5} & \multicolumn{1}{c}{60.6} \\

        h & \multicolumn{1}{c}{V} & \multicolumn{1}{c}{$\checkmark$} & \multicolumn{1}{c}{$\checkmark$} & \multicolumn{1}{c}{$\checkmark$} & \multicolumn{1}{c}{$\times$} & \multicolumn{1}{c}{$\checkmark$} & \multicolumn{1}{c}{89.1} & \multicolumn{1}{c}{84.4} & \multicolumn{1}{c}{84.9} & \multicolumn{1}{c}{80.9} & \multicolumn{1}{c}{78.8} & \multicolumn{1}{c}{65.3} \\

        i & \multicolumn{1}{c}{V} & \multicolumn{1}{c}{$\checkmark$} & \multicolumn{1}{c}{$\checkmark$} & \multicolumn{1}{c}{$\checkmark$} & \multicolumn{1}{c}{$\checkmark$} & \multicolumn{1}{c}{$\checkmark$} & \multicolumn{1}{c}{91.8} & \multicolumn{1}{c}{90.6} & \multicolumn{1}{c}{85.0} & \multicolumn{1}{c}{81.0} & \multicolumn{1}{c}{79.1} & \multicolumn{1}{c}{66.1} \\

        j & \multicolumn{1}{c}{V+C} & \multicolumn{1}{c}{$\checkmark$} & \multicolumn{1}{c}{$\checkmark$} & \multicolumn{1}{c}{$\checkmark$} & \multicolumn{1}{c}{$\checkmark$} & \multicolumn{1}{c}{$\checkmark$} & \multicolumn{1}{c}{\textbf{93.5}} & \multicolumn{1}{c}{\textbf{92.3}} & \multicolumn{1}{c}{86.7} & \multicolumn{1}{c}{85.0} & \multicolumn{1}{c}{\textbf{80.2}} & \multicolumn{1}{c}{\textbf{69.6}} \\

        h & \multicolumn{1}{c}{V+C+S} & \multicolumn{1}{c}{$\checkmark$} & \multicolumn{1}{c}{$\checkmark$} & \multicolumn{1}{c}{$\checkmark$} & \multicolumn{1}{c}{$\checkmark$} & \multicolumn{1}{c}{$\checkmark$} & \multicolumn{1}{c}{90.5} & \multicolumn{1}{c}{89.6} & \multicolumn{1}{c}{\textbf{87.7}} & \multicolumn{1}{c}{\textbf{85.4}} & \multicolumn{1}{c}{79.7} & \multicolumn{1}{c}{69.0} \\

        \bottomrule
	\end{tabular*}
	\label{table:ablation}
\end{table*}

\subsubsection{Component Effectiveness}
\label{sec:com}
We examine the effect of the proposed components: \( G \), \( R \), \( D_{adv} \), \( D_{diff} \), \( D_{rep} \), and \( \mathcal{L}_{mu} \). The results are reported in Table~\ref{table:ablation}, and we could get three conclusion:
\begin{enumerate}
    \item DiffAug alone can improve the T1 of AWA2 from 76.0\% to 80.0\%. Further addition of SC and $\mathcal{L}_{mu}$ further increases the T1 to 86.4, and the H value increases accordingly. This proves the effectiveness of our three training insights on stabilizing the visual-semantic association during the training phase.
    \item While DiffTTA offers limited improvement over a single FNGen when used for generation alone, it significantly improves the seen/unseen balance when used to adapt the generator and combined with DiffGen. This indicates that fine-tuning the generator using the reconstruction loss of real test features during testing helps narrow the training/test distribution gap.
    \item DiffGen introduces a traceable and partially imaginable generation strategy, which in itself brings significant performance improvements. As shown in the table, it achieves optimal performance when used in conjunction with DiffTTA. This demonstrates the effectiveness of generation methods based on real-world noise features compared to pure noise generation, significantly narrowing the gap between generated features and real test features.
\end{enumerate}

%We observe that \( D_{diff} \) and \( \mathcal{L}_{mu} \) consistently improve the performance across the three datasets. Another observation is that, in many cases, \( G \) benefits from the inclusion of \( R \), although \( R \) does not perform better than \( G \) alone.
%It also evidences our motivation that SC-based representations could capture instance-level semantics to support feature generation.

\subsubsection{Ablation study for the number of generated features}
\label{sec:hyperpara}
For hyper-parameter sensitivity, we conclude the experiments about the number of synthesized samples for each unseen class $n_{syn}$ in  Fig.~\ref{fig:ablation_n_syn}.
The accuracy for the unseen class increases with the number of synthesized samples, and when $N_{syn}$ is large enough, the performances become stable. 
This result demonstrates that the features synthesized by our method effectively mitigate the issue of missing data for unseen classes.

\begin{figure}
    \centering
    \includegraphics[width=\linewidth]{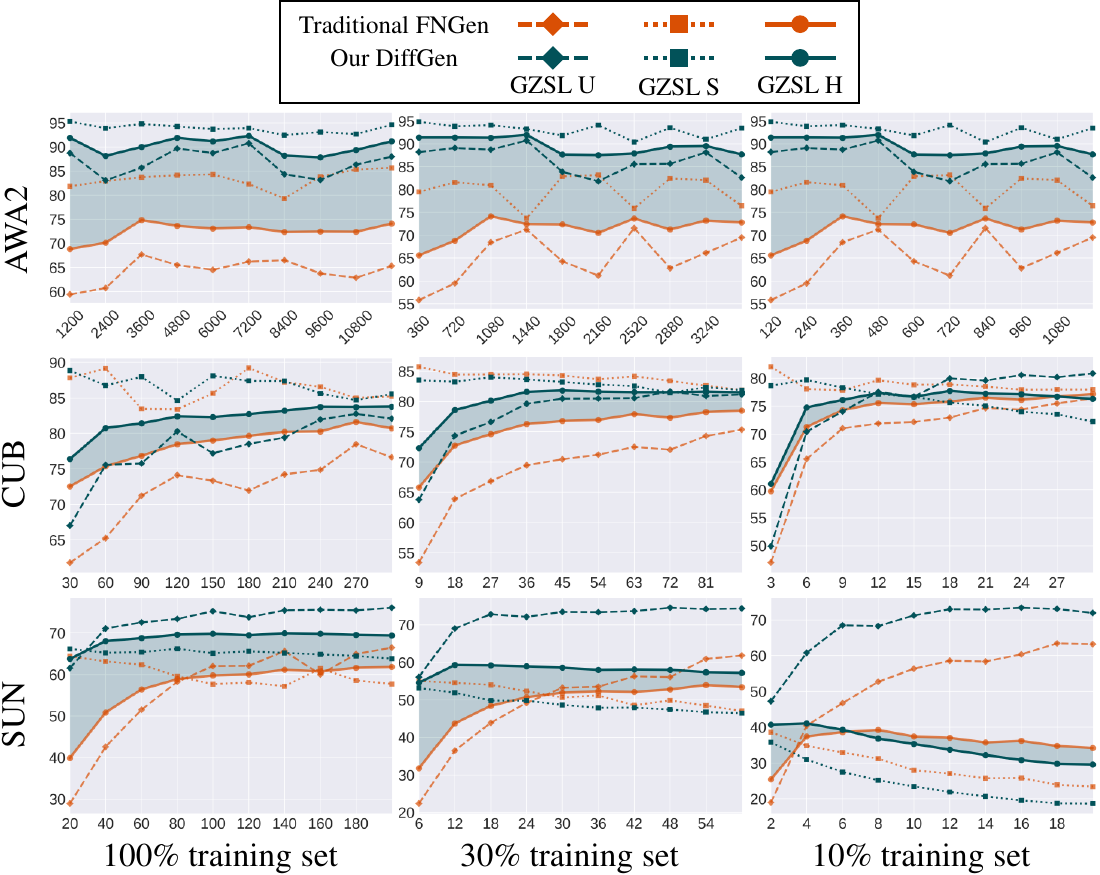}
    \caption{The effect of the number of synthetic samples $N_{syn}$. The shaded area indicates the improvement of GZSL H from our DiffGen compared to the traditional FNGen.}
    \label{fig:ablation_n_syn}
\end{figure}

\subsubsection{Ablation study for DiffGen}
\label{sec:ablation_diffgen}

We ablate the design choices of DiffGen along two axes: the type of SC (paired semantics) fed to the generator and the diffusion time $t$ used for partial denoising. Table~\ref{table:ablation_diffgen} reports results for these variants. We have three main findings.
\begin{enumerate}
    \item Real SC vs. Fake SC. Replacing the generated SC with the real, paired SC consistently improves both ZSL T1 and GZSL H across datasets. For instance, on AWA2, using the real SC at t=2 yields ZSL T1 = 92.3\% and GZSL H = 91.9\%, while the best “Fake” SC configuration reaches only H = 81.8\% (AWA2, t=1). This confirms that using paired instance-level semantics strongly tightens the visual–semantic link and reduces generation drift.
    \item A trade-off between copy and imagination. Small $t$ means DiffGen copies more of the real feature; large $t$ approaches fully-noised, fully-synthetic generation. The table shows a clear sweet-spot at intermediate $t$ (e.g., $t=2$): AWA2 H peaks at 91.9\% for real SC at $t=2$, while very small or maximal $t$ either under-utilize imagination or reintroduce distribution mismatch. Thus, partial denoising with a moderate $t$ balances realism and diversity and gives the best seen/unseen harmonic means.
    \item Randomized $t$ in DiffGen is acceptable but suboptimal. Randomizing DiffGen $t$ increases robustness to step choice and can sometimes improve individual dataset metrics by introducing a mixture of copy/imagine behaviors, but it is less consistent than fixing $t$ at the empirically chosen moderate value. In our experiments, while random $t$ occasionally matches or slightly exceeds single-metric scores on CUB or SUN, the stable best-performing configuration remains real paired SC + fixed moderate $t$ for balanced GZSL H.
\end{enumerate}

\begin{table*}[htbp]
	\centering
	\caption{Ablation study on the different diffusion time $t$ and the SC type of DiffGen (Eq.~\ref{eq:zerodiffpp_gen}).}
    \begin{tabular*}{\textwidth}{@{\extracolsep\fill}ccccccccccccccc}
 
		\toprule
        
        \multirow{3}{*}{SC type} & \multirow{3}{*}{$t$} & \multirow{3}{*}{$T$} & \multicolumn{3}{{@{}c@{}}}{ZSL} & \multicolumn{9}{c}{GZSL} \\
        \cmidrule{4-6}
        \cmidrule{7-15}
        
		& & & \multicolumn{1}{c}{AWA2} & \multicolumn{1}{c}{CUB} & \multicolumn{1}{c}{SUN} & \multicolumn{3}{{@{}c@{}}}{AWA2} & \multicolumn{3}{c}{CUB}  & \multicolumn{3}{c}{SUN} \\
        \cmidrule{4-4}
        \cmidrule{5-5}
        \cmidrule{6-6}
        \cmidrule{7-9}
        \cmidrule {10-12}
        \cmidrule{13-15}
        
        & & & \multicolumn{1}{c}{T1} & \multicolumn{1}{c}{T1} & \multicolumn{1}{c}{T1} & \multicolumn{1}{c}{U} & \multicolumn{1}{c}{S} & \multicolumn{1}{c}{H} & \multicolumn{1}{c}{U}& \multicolumn{1}{c}{S} & \multicolumn{1}{c}{H} & \multicolumn{1}{c}{U} & \multicolumn{1}{c}{S} & \multicolumn{1}{c}{H} \\
        \midrule

        \multirow{5}{*}{Fake: $\tilde{\mathbf{r}}^{u}_{0}$} & \multicolumn{1}{c}{1} & \multicolumn{1}{c}{4} & \multicolumn{1}{c}{88.9} & \multicolumn{1}{c}{87.7} & \multicolumn{1}{c}{79.2} & \multicolumn{1}{c}{76.5} & \multicolumn{1}{c}{88.0} & \multicolumn{1}{c}{81.8} & \multicolumn{1}{c}{80.7} & \multicolumn{1}{c}{84.0} & \multicolumn{1}{c}{82.3} & \multicolumn{1}{c}{71.0} & \multicolumn{1}{c}{61.3} & \multicolumn{1}{c}{65.8} \\
        
       & \multicolumn{1}{c}{2} & \multicolumn{1}{c}{4} & \multicolumn{1}{c}{87.1} & \multicolumn{1}{c}{87.3} & \multicolumn{1}{c}{79.9} & \multicolumn{1}{c}{68.7} & \multicolumn{1}{c}{85.9} & \multicolumn{1}{c}{76.3} & \multicolumn{1}{c}{81.0} & \multicolumn{1}{c}{84.1} & \multicolumn{1}{c}{82.6}  & \multicolumn{1}{c}{68.9} & \multicolumn{1}{c}{58.7} & \multicolumn{1}{c}{63.4} \\

        & \multicolumn{1}{c}{3} & \multicolumn{1}{c}{4} & \multicolumn{1}{c}{81.4} & \multicolumn{1}{c}{87.6} & \multicolumn{1}{c}{75.4} & \multicolumn{1}{c}{62.4} & \multicolumn{1}{c}{87.0} & \multicolumn{1}{c}{72.7} & \multicolumn{1}{c}{82.0} & \multicolumn{1}{c}{81.8} & \multicolumn{1}{c}{81.9} & \multicolumn{1}{c}{64.9} & \multicolumn{1}{c}{58.1} & \multicolumn{1}{c}{61.3} \\

        & \multicolumn{1}{c}{4} & \multicolumn{1}{c}{4} & \multicolumn{1}{c}{82.5} & \multicolumn{1}{c}{87.3} & \multicolumn{1}{c}{75.4} & \multicolumn{1}{c}{62.4} & \multicolumn{1}{c}{86.4} & \multicolumn{1}{c}{72.5} & \multicolumn{1}{c}{78.7} & \multicolumn{1}{c}{84.3} & \multicolumn{1}{c}{81.4} & \multicolumn{1}{c}{62.0} & \multicolumn{1}{c}{57.9} & \multicolumn{1}{c}{59.9} \\

        & \multicolumn{1}{c}{Random} & \multicolumn{1}{c}{4} & \multicolumn{1}{c}{86.7} & \multicolumn{1}{c}{87.7} & \multicolumn{1}{c}{77.3} & \multicolumn{1}{c}{73.5} & \multicolumn{1}{c}{87.7} & \multicolumn{1}{c}{80.0} & \multicolumn{1}{c}{80.5} & \multicolumn{1}{c}{85.3} & \multicolumn{1}{c}{82.8} & \multicolumn{1}{c}{66.3} & \multicolumn{1}{c}{62.1} & \multicolumn{1}{c}{64.1} \\

        \midrule
        
        \multirow{5}{*}{Real: $\mathbf{r}^{u}_{0}$} & \multicolumn{1}{c}{1} & \multicolumn{1}{c}{4} & \multicolumn{1}{c}{91.0} & \multicolumn{1}{c}{87.7} & \multicolumn{1}{c}{79.5} & \multicolumn{1}{c}{88.6} & \multicolumn{1}{c}{93.9} & \multicolumn{1}{c}{91.2} & \multicolumn{1}{c}{84.2}  & \multicolumn{1}{c}{86.1} & \multicolumn{1}{c}{85.1} & \multicolumn{1}{c}{78.3} & \multicolumn{1}{c}{63.3} & \multicolumn{1}{c}{70.0} \\

       & \multicolumn{1}{c}{2} & \multicolumn{1}{c}{4} & \multicolumn{1}{c}{92.3} & \multicolumn{1}{c}{87.3} & \multicolumn{1}{c}{79.8} & \multicolumn{1}{c}{89.8} & \multicolumn{1}{c}{94.2} & \multicolumn{1}{c}{91.9} & \multicolumn{1}{c}{83.9}  & \multicolumn{1}{c}{85.9} & \multicolumn{1}{c}{84.9} & \multicolumn{1}{c}{78.0} & \multicolumn{1}{c}{63.4} & \multicolumn{1}{c}{69.9} \\

        & \multicolumn{1}{c}{3} & \multicolumn{1}{c}{4} & \multicolumn{1}{c}{92.1} & \multicolumn{1}{c}{87.6} & \multicolumn{1}{c}{80.0} & \multicolumn{1}{c}{88.8} & \multicolumn{1}{c}{94.1} & \multicolumn{1}{c}{91.4} & \multicolumn{1}{c}{85.1}  & \multicolumn{1}{c}{84.6} & \multicolumn{1}{c}{84.9} & \multicolumn{1}{c}{78.6} & \multicolumn{1}{c}{63.9} & \multicolumn{1}{c}{70.5} \\

        & \multicolumn{1}{c}{4} & \multicolumn{1}{c}{4} & \multicolumn{1}{c}{91.7} & \multicolumn{1}{c}{87.3} & \multicolumn{1}{c}{80.2} & \multicolumn{1}{c}{88.6} & \multicolumn{1}{c}{94.1} & \multicolumn{1}{c}{91.3} & \multicolumn{1}{c}{84.1}  & \multicolumn{1}{c}{85.9} & \multicolumn{1}{c}{85.0} & \multicolumn{1}{c}{77.9} & \multicolumn{1}{c}{63.3} & \multicolumn{1}{c}{69.8} \\

        & \multicolumn{1}{c}{Random} & \multicolumn{1}{c}{4} & \multicolumn{1}{c}{92.2} & \multicolumn{1}{c}{87.7} & \multicolumn{1}{c}{78.6} & \multicolumn{1}{c}{88.9} & \multicolumn{1}{c}{94.6} & \multicolumn{1}{c}{91.6} & \multicolumn{1}{c}{84.4}  & \multicolumn{1}{c}{86.4} & \multicolumn{1}{c}{85.4}  & \multicolumn{1}{c}{76.7} & \multicolumn{1}{c}{63.5} & \multicolumn{1}{c}{69.5} \\

        \bottomrule
	\end{tabular*}
	\label{table:ablation_diffgen}
\end{table*}

\subsection{Mutual Learning Effectiveness}
\label{sec:mu_effect}
To further verify the effect of our \( \mathcal{L}_{mu} \), we designed an additional experiment to show the change in critic score.
Due to the page limitation, we provide the results and discussion in Appendix~\ref{app:mu_effect}.

\subsection{Visualization Analysis}

\subsubsection{Feature TSNE Visualization}

\begin{figure}
    \centering
    \includegraphics[width=\linewidth]{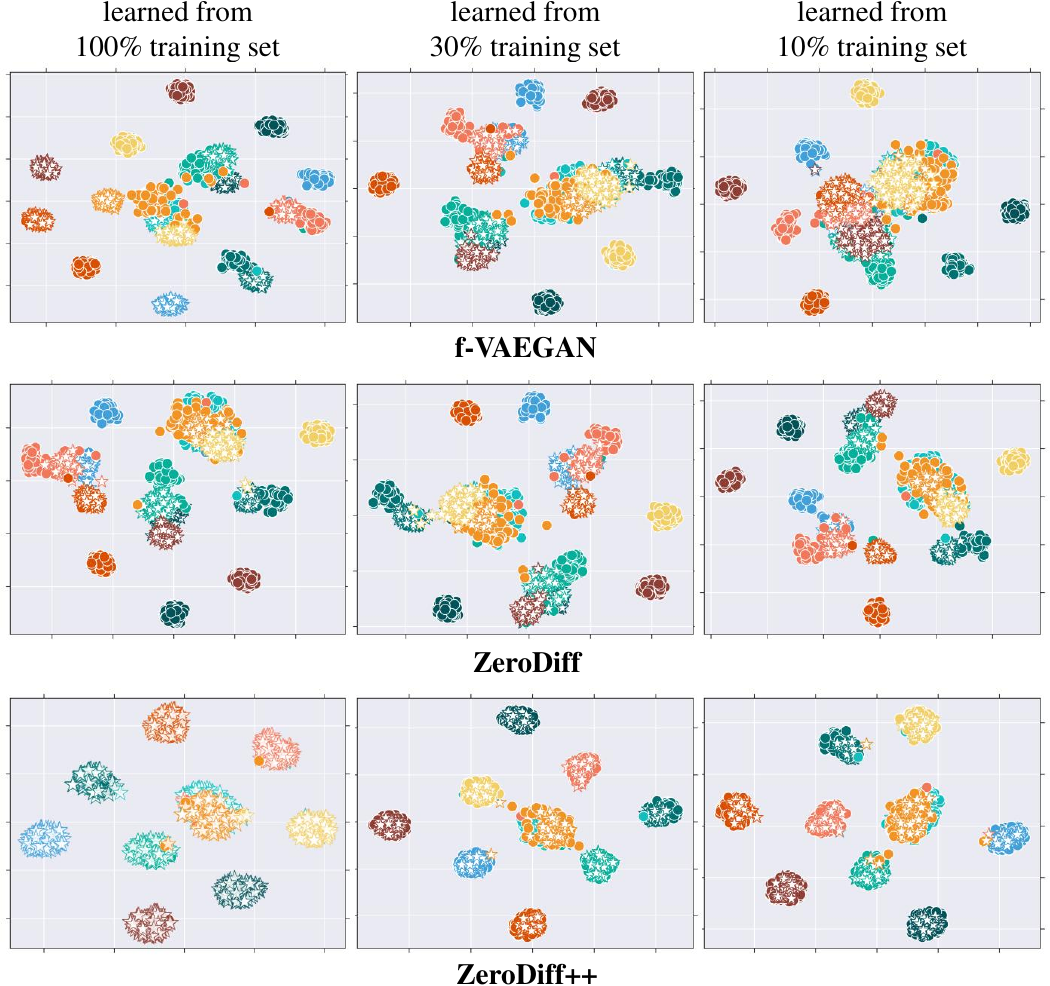}
    \caption{Qualitative evaluation with t-SNE visualization on AWA2. We randomly selected 100 generated samples per class and real samples for 10 unseen classes. We use different colors to denote classes, and use solid \( \bullet \) and hollow \(  \bigstar \)  to denote the real and synthesized sample features, respectively (Best Viewed in Color).
}
    \label{fig:vis_feat}
\end{figure}
We also provide a qualitative comparison between the baseline f-VAEGAN, ZeroDiff, and our ZeroDiff++ using t-SNE visualization of the real and synthesized sample features in Fig.~\ref{fig:vis_feat}.
The visualization shows that as the number of training samples decreases, f-VAEGAN gradually fails to generate unseen classes fully.
The features generated by fvaegan using 10\% of the training data clustered together and could not be classified at all.
Zerodiff, on the other hand, has relative robustness and generates features with good class separation, but it differs significantly from the true features.
Finally, the features generated by our zerodiff++ are not only highly robust, achieving class separation with only 10\% training time, but also closely approximate the true feature distribution. This further highlights the role of DiffTTA (unseen class adaptation) and DiffGen (real-feature-based generation).

\begin{figure}
    \centering
    \includegraphics[width=\linewidth]{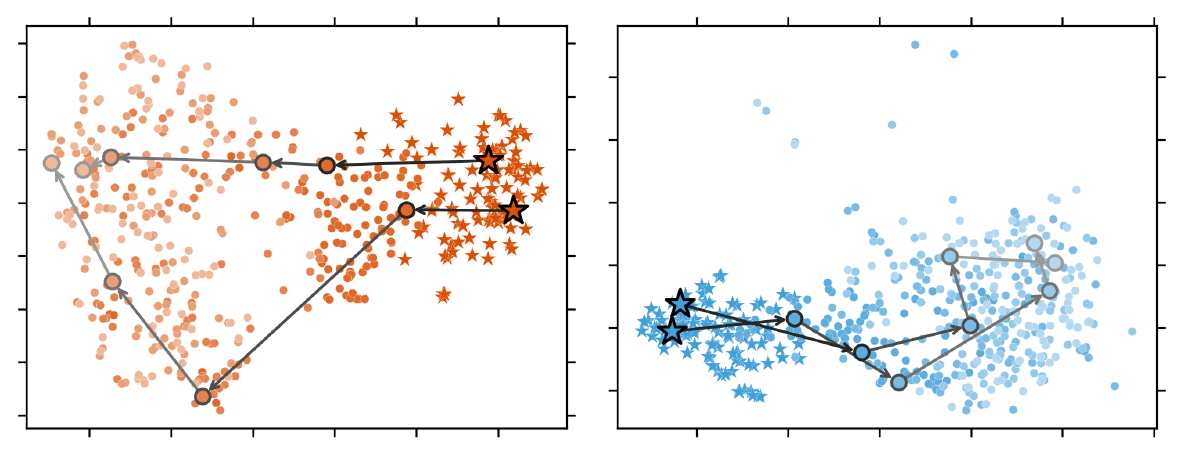}
    \caption{The visualization for the traceable generation path of DiffGen on two unseen classes from AWA2. We use different colors to denote classes, and use  \( \star \)  and \( \bullet \) to denote the real and synthesized sample features, respectively. The lighter the color, the larger the diffusion step, resulting in a greater distance from the original features (Best Viewed in Color).}
    \label{fig:vis_diffgen}
\end{figure}

\subsubsection{Traceable Generation Path}
We also visualize the generation path of our DiffGen that connects each synthesized feature to a specific real test feature sample via diffusing the real feature at diffusion time $t$ and denoising it partially.
The traceability figure is Fig.~\ref{fig:vis_diffgen}, which means a smaller diffusion time $t$ follows a denoising trajectory toward real features.
It shows that darker colored points indicate longer diffusion times, a lower proportion of real features in the generated features, and a greater distance between the generated feature points and the original feature points. This not only further verifies the partial feature imagination effectiveness of our DiffGen , i.e., injecting real features can effectively alleviate the mismatch between generated and real features, but also provides a traceable generation pattern, indicating that the generated points smoothly transition from clean to noise along the chain, thus facilitating error analysis and selective filtering.

\section{Theoretical Perspective}
\label{sec:theory}

\subsection{Theory of Overlap Mass}

\subsubsection{Overlap Mass Definition}\label{sec:overlapmass}

Recall the forward diffusion process (Eq.~\ref{eq:diff_forward_chain} and \ref{eq:diff_forward_chain_step}).
Due to the Markov property of the forward process, we have the marginal distribution of $\mathbf{x}_t$ given the initial clean data $\mathbf{x}_0$:
\begin{equation}
\label{eq:x0_to_xt}
q(\mathbf{x}_t|\mathbf{x}_{0}) = \mathcal{N}(\mathbf{x}_{t}; \sqrt{\bar{\alpha}} \mathbf{x}_{0}, (1-\bar{\alpha}) \bm{I}),
\end{equation}
where we denote $\alpha_t:=1-\beta_t$ and $\bar{\alpha}_{t} := \prod^{t}_{t^\prime=1} \alpha_{t^\prime}$.

\begin{figure}
\centering
\includegraphics[width=0.8\linewidth]{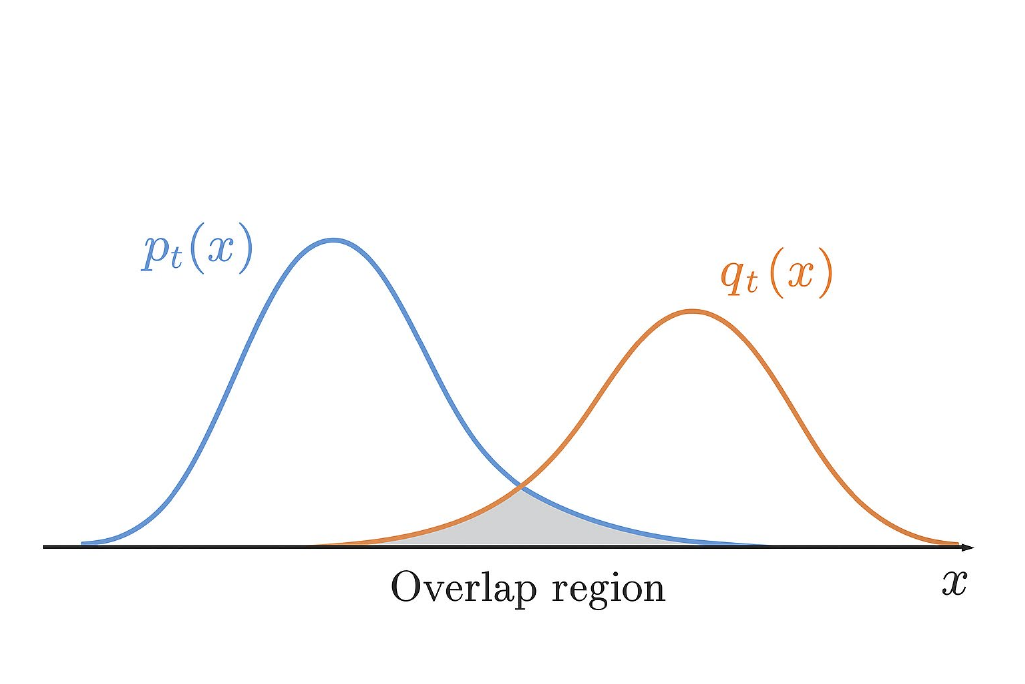} 
\caption{An illustration about the overlap mass under 1-D data.}
\label{fig:overlap_prob}
\end{figure}

We can define the \textbf{overlap mass} as
\begin{equation}
\boxed{\;
\mathcal{M}(t)\;=\;\int\min\{p_t(x),q_t(x)\}\,dx\;.}
\end{equation}
Intuitively, as shown in Fig.~\ref{fig:overlap_prob}, we can use $\min\{p_t(x),q_t(x)\}$ represent the probability density of the point $x$ being covered by two distributions simultaneously, and integrate $\int dx$ accumulates the common coverage probability of each point to obtain the total overlap mass, which is the sum of the probabilities of the overlapping regions of the two distributions.
If $\mathcal{M}(t)=0$ (i.e., two distributions are disjoint supports), the discriminator can perfectly separate the two distributions and relevant divergences (e.g., Jensen–Shannon Divergence, JSD) may saturate, causing generator gradients to vanish. In contrast, $\mathcal{M}(t)>0$ implies that there exist ambiguous or \emph{middle} samples where both distributions place positive mass, and the discriminator cannot saturate everywhere. Prior works note that smoothing by noise produces such overlap and prevents JSD saturation\cite{arjovsky2017towards, wang2022diffusion}.

\subsubsection{Overlap Mass Under One-dimensional Case}

To obtain an analytic expression for intuition, consider the 1D Dirac case where the two initial point masses lie at $0$ and $m$ (i.e., $p_0=\delta_0$ and $q_0=\delta_m$), respectively. After $t$ steps of the diffusion defined above, these map to Gaussians
\begin{align}
p_t(x)&= \mathcal{N}\bigl(x;0,\;\sigma^2 = 1-\bar{\alpha}_{t}\bigr),  \\
q_t(x)&= \mathcal{N}\bigl(x;\sqrt{\bar{\alpha}_t}m,\;\sigma^2 = 1-\bar{\alpha}_{t}\bigr).
\end{align}
The two density curves intersect at \(x=\tfrac{1}{2}\sqrt{\bar{\alpha}_t}m\).
The overlap mass has the closed form
\begin{align}
\mathcal{M}(t)
& = \int \min\{p_t(x),q_t(x)\}\,\mathrm{d}x \nonumber \\
& = 2\Bigl(1-\Phi\!\Bigl(\frac{|m|\sqrt{\bar{\alpha}_t}}{2\sqrt{1-\bar{\alpha}_t}}\Bigr)\Bigr)  \\
& = 2\Phi\!\Bigl(-\frac{|m|\sqrt{\bar{\alpha}_t}}{2\sqrt{1-\bar{\alpha}_t}}\Bigr)
= 2\Phi\Bigl(-\frac{|m|\sqrt{\bar{\alpha}_t}}{2\sigma}\Bigr),  \nonumber
\end{align}
where \(\Phi(\cdot)\) denotes the standard normal Cumulative Distribution Function (CDF).
This formula indicates that as the diffusion degree increases (\(\bar{\alpha}_t \to 0\)), the distributions become closer to each other and \(\mathcal{M}(t) \to 2\Phi(0)=1\). When the noise is small (\(\bar{\alpha}_t \approx 1\)) or the initial separation \(|m|\) is large, the overlap \(\mathcal{M}(t)\) is near zero, which corresponds to gradient saturation.

\subsubsection{Overlap mass in the multi-dimensional case}
Now consider the case in \(\mathbb{R}^d\).
Assume the two post-diffusion marginals are Gaussians with means \(\mu_{p},\mu_{q}\) and common isotropic covariance \((1-\bar{\alpha}_{t})I\):
\begin{equation}
    p_t = \mathcal{N}(\mu_p,(1-\bar{\alpha}_t)I),\qquad q_t = \mathcal{N}(\mu_q,(1-\bar{\alpha}_t)I).
\end{equation}
Let \(\Delta_{\mu} = \mu_p - \mu_q\).
The Mahalanobis distance $D$ squared between the means is
\begin{equation}
    D^2 = \Delta_{\mu}^\top\bigl[(1-\bar{\alpha}_t)I\bigr]^{-1}\Delta_\mu
    = \frac{\|\Delta_\mu\|^2}{1-\bar{\alpha}_{t}}.
\end{equation}
Because \(p_t\) and \(q_t\) share the same covariance, the Bhattacharyya Coefficient (BC) admits the closed form
\begin{equation}
    BC := \int\sqrt{p_t(x)q_t(x)}\,\mathrm{d}x = \exp\!\Bigl(-\frac{1}{8}D^2\Bigr).
\end{equation}
Noting \(\sqrt{p_t q_t} = \sqrt{\min\{p_t,q_t\}\cdot\max\{p_t,q_t\}}\), by Cauchy–Schwarz we have
\begin{align}
    BC &= \int\sqrt{\min\{p_t,q_t\}\cdot\max\{p_t,q_t\}}\,\mathrm{d}x \\
    &\le \sqrt{\Bigl(\int\min\{p_t,q_t\}\Bigr)\Bigl(\int\max\{p_t,q_t\}\Bigr)}.
\end{align}
Let \(\mathcal{M}(t):=\int\min\{p_t,q_t\}\). Then \(\int\max\{p_t,q_t\}=2-\mathcal{M}(t)\), and hence
\begin{equation}
    BC^2 \le \mathcal{M}(t)(2-\mathcal{M}(t)) \quad\Longrightarrow\quad ( \mathcal{M}(t)-1 )^2 \le 1-BC^2.
\end{equation}
Solving for \(\mathcal{M}(t)\) gives the exact (but slightly implicit) lower bound
\begin{equation}
    \mathcal{M}(t) \ge 1 - \sqrt{1-BC^2}.
\end{equation}
Using the elementary inequality \(\sqrt{1-a}\le 1-\tfrac{a}{2}\) for \(a\in[0,1]\) and substituting \(BC^2=\exp(-D^2/4)\) yields the simpler conservative bound
\begin{equation}
    \boxed{\;
    \mathcal{M}(t) \ge \tfrac{1}{2}\exp\!\Bigl(-\frac{D^2}{4}\Bigr)
    \;=\; \tfrac{1}{2}\exp\!\Bigl(-\frac{\|\Delta_\mu\|^2}{4(1-\bar{\alpha}_t)}\Bigr)\;.}
\end{equation}

If the clean mean difference is \(\Delta_0\), then the noised mean difference is \(\Delta_\mu=\sqrt{\bar{\alpha}_t}\,\Delta_0\), giving
\begin{equation}
    D^2 = \frac{\bar{\alpha}_t \|\Delta_0\|^2}{1-\bar{\alpha}_t},
\end{equation}
and therefore
\begin{equation}
    \mathcal{M}(t) \ge \tfrac{1}{2}\exp\!\Bigl(-\frac{\bar{\alpha}_t\|\Delta_0\|^2}{4(1-\bar{\alpha}_t)}\Bigr).
\end{equation}
This inequality highlights that increasing diffusion (reducing \(\bar{\alpha}_t\)) reduces the effective Mahalanobis distance and increases the lower bound of overlap mass. In the limit \(t\to T\) (so \(\bar{\alpha}_t\to 0\)), the overlap lower bound tends to a constant and the actual \(\mathcal{M}(t)\to 1\).

\subsubsection{Implications for discriminator gradients}
The overlap bounds above explain why diffusion-based discriminators stabilize GAN training. Whenever $p_t$ and $q_t$ overlap, the optimal discriminator at time $t$ is
\(
    D^*_{diff} = \frac{p_t(x)}{p_t(x)+q_t(x)},
\)
which lies strictly between 0 and 1 on any point $x$ where both densities are positive.
Therefore, loss derivatives remain nonzero and provide meaningful feedback to the generator parameters.
From the overlap-mass perspective developed above, at each diffusion step $t$, the real and fake marginals necessarily share support, so the discriminator loss is a smooth, well-behaved function of the generator.
These overlap guarantees imply that discriminator gradients do not vanish under diffusion-based smoothing, theoretically proving that diffusion-based discriminators are a remedy for disjoint-support failure modes in GAN training.

% =========================
% Theory Supplement for Overlap Mass (rewritten to match current notation)
% =========================
% This subsection is designed to directly support and complement
% Sec. "Theory of Overlap Mass" under the same diffusion notation
% (beta_t, alpha_t, \bar{\alpha}_t, q(v_t|v_{t-1}), q(v_t|v_0)).

\subsection{Diffusion marginalization and overlap mass: an information-theoretic complement}
\label{sec:theory_overlap_kl}

In Sec.~\ref{sec:overlapmass}, we analyzed the \emph{overlap mass}
\begin{equation}
\mathcal{M}(t)\;=\;\int \min\{p_t(x),q_t(x)\}\,dx,
\end{equation}
under simplified distributional assumptions, showing that the forward diffusion
process increases distributional overlap. Here we provide a more general
information-theoretic complement: the same diffusion marginalization
\emph{contracts distribution discrepancy} in the KL divergence, which supports the
overlap-mass intuition beyond specific parametric forms.

\paragraph{Common diffusion marginalization under the DDPM forward process.}
Consider the forward diffusion process used throughout this paper:
\begin{align}
q(\mathbf{v}_{1:T}|\mathbf{v}_0) &= \prod_{t=1}^{T} q(\mathbf{v}_{t}|\mathbf{v}_{t-1}),q(\mathbf{v}_t|\mathbf{v}_{t-1}) \nonumber \\
&= \mathcal{N}\!\left(\mathbf{v}_{t}; \sqrt{1-\beta_{t}}\mathbf{v}_{t-1}, \beta_{t} \bm{I}\right),
\end{align}
with $\alpha_t:=1-\beta_t$ and $\bar{\alpha}_{t} := \prod^{t}_{s=1} \alpha_{s}$.
By the Markov property, the marginal distribution admits a closed form
\begin{equation}
\label{eq:kl_supp_x0_to_xt}
q(\mathbf{v}_t|\mathbf{v}_{0}) = \mathcal{N}\!\left(\mathbf{v}_{t}; \sqrt{\bar{\alpha}_t}\mathbf{v}_{0}, (1-\bar{\alpha}_t)\bm{I}\right).
\end{equation}

Let $p_0$ and $q_0$ denote two probability measures over clean features $\mathbf{v}_0$
(e.g., real and generated feature distributions). We define their diffused
marginals at step $t$ by pushing forward through the same kernel
$q(\mathbf{v}_t|\mathbf{v}_0)$:
\begin{align}
&p_t(\mathbf{v}_t) := \int q(\mathbf{v}_t|\mathbf{v}_0)\,p_0(\mathbf{v}_0)\,d\mathbf{v}_0,
\qquad \\
&q_t(\mathbf{v}_t) := \int q(\mathbf{v}_t|\mathbf{v}_0)\,q_0(\mathbf{v}_0)\,d\mathbf{v}_0.
\end{align}
Equivalently, if $\mathbf{v}_t=\sqrt{\bar{\alpha}_t}\mathbf{v}_0+\sqrt{1-\bar{\alpha}_t}\,\bm{\epsilon}$
with $\bm{\epsilon}\sim\mathcal{N}(0,I)$, then $p_t$ and $q_t$ are the corresponding
$t$-marginal laws.

\paragraph{KL contraction under diffusion marginalization.}
A fundamental result in \cite{raginsky2016strong,luo2023training,wang2025uni} shows that marginalization along a
common diffusion process is a \emph{contraction mapping} for KL divergence.
Specialized to the forward diffusion marginalization above, we have
\begin{equation}
\label{eq:kl_contraction_overlap_new}
D_{\mathrm{KL}}(p_t \,\|\, q_t)\ \le\ D_{\mathrm{KL}}(p_0 \,\|\, q_0),\quad \forall\, t\ge 0.
\end{equation}
Moreover, in continuous-time diffusion formulations, the KL evolution admits a
dissipation identity of the form \cite{luo2023training}:
\begin{align}
\label{eq:kl_dissipation_overlap_new}
&\frac{d}{dt} D_{\mathrm{KL}}(p_t\|q_t)  =  \\
&
-\frac{1}{2}\,
\mathbb{E}_{\mathbf{v}_t\sim p_t}
\!\left[
G^2(t)\,
\bigl\|
\nabla_{\mathbf{v}_t} \log p_t(\mathbf{v}_t)
-
\nabla_{\mathbf{v}_t} \log q_t(\mathbf{v}_t)
\bigr\|_2^2
\right]
\le 0,\nonumber 
\end{align}
showing that the distributional discrepancy monotonically decreases along the
diffusion trajectory.

\paragraph{Connection to overlap mass.}
Overlap mass $\mathcal{M}(t)$ measures similarity from a geometric viewpoint,
whereas the KL divergence quantifies discrepancy from an information-theoretic
viewpoint. The contraction in \eqref{eq:kl_contraction_overlap_new} implies that
the forward diffusion kernel in \eqref{eq:kl_supp_x0_to_xt} progressively smooths
both $p_0$ and $q_0$, removing sharp mismatches and reducing their information
gap. Consequently, the diffused marginals $p_t$ and $q_t$ become increasingly
entangled in support, which is consistent with the overlap-mass increase derived
in Sec.~\ref{sec:overlapmass}. Importantly, unlike the Gaussian single-mode
Assumptions used for closed-form overlap calculations, the KL contraction result
holds under much milder regularity conditions.

\paragraph{Implication for diffusion-augmented discrimination.}
In ZeroDiff++, discriminators compare real and generated features at multiple
diffusion steps. Equations \eqref{eq:kl_contraction_overlap_new}--\eqref{eq:kl_dissipation_overlap_new}
indicate that larger diffusion steps correspond to regimes where distributional
discrepancy is provably smaller, preventing discriminators from exploiting
low-overlap artifacts in the clean feature space. Thus, diffusion augmentation
stabilizes discriminative learning by enforcing alignment across progressively
higher-overlap (lower-discrepancy) regimes, complementing and strengthening the
overlap-mass analysis in Sec.~\ref{sec:overlapmass}.

\subsection{Theory of Overfitting Reducing}

\subsubsection{Notation.}
Now, we prove the reduction of discriminator overfitting caused by the diffusion process.
We write $P_X := \mathcal L(X)$ for the probability measure of a random variable $X$ and $p_x$ is the probability density function of $x$ (so $p_x=\mathrm dP/\mathrm d x$).
For a measurable map $T:\mathbb R^d\to\mathbb R^d$ and measure $P$ we write $T_{\#}P$ for the pushforward of $P$ by $T$.
The Lipschitz semi-norm is denoted
\[
\|f\|_{\mathrm{Lip}} := \sup_{x\neq y}\frac{|f(x)-f(y)|}{\|x-y\|},
\]
and we say that $f$ is $L$-Lipschitz when $\|f\|_{\mathrm{Lip}}\le L$.
We use the Euclidean norm $\|\cdot\|$ on $\mathbb R^d$.

\subsubsection{Wasserstein Distance Contraction}

Let $P_{tr},P_{val}$ denote the probability measures for the clean training set and the clean validation set, respectively. Assume they have finite first moments: $\int \|x\|\,\mathrm{d}P_{*}(x)<\infty$. Fix a diffusion step $t$ and write $\bar\alpha_t=\prod_{s=1}^t(1-\beta_s)\in(0,1]$. Define the forward-diffused random variable
$ x_t=\sqrt{\bar\alpha_t}\,x_0 + \sqrt{1-\bar\alpha_t}\,z$ and $ z\sim\mathcal N(0,I)$,
where $x_0$ is a clean sample. Denote by $\nu:=\mathcal N(0,(1-\bar\alpha_t)I)$. The probability measures of the diffused data distribution are then
\[
P_{tr,t} := P_{tr} * \nu,\qquad P_{val,t} := P_{val} * \nu,
\]
i.e.\ convolution of the clean laws with the same Gaussian kernel. We consider discriminator functions $D:\mathbb R^d\to\mathbb R$ that are approximately $1$-Lipschitz since we use the WGAN-GP gradient penalty, which enforces a near-Lipschitz constraint. We assume discriminators are trained sufficiently on their respective training sets so that training error is negligible.

\begin{shaded}
\begin{lemma}[Kantorovich--Rubinstein duality]\label{lem:KR}
For any probability measures $P,Q$ with finite first moments,
\[
W_1(P,Q) = \sup_{\|f\|_{\mathrm{Lip}}\le 1}\Big\{E_{P}[f]-E_{Q}[f]\Big\}.
\]
\end{lemma}
\end{shaded}

\begin{proof}[Comment]
This is the classical Kantorovich--Rubinstein duality. We use it to upper-bound expectation differences of any $1$-Lipschitz function by $W_1$.
\end{proof}

\begin{shaded}
\begin{theorem}[W-distance Contraction]\label{thm:W-contraction}
Let $S(x):=\sqrt{\bar\alpha_t}\,x$ be the linear scaling map. Define $P':=S_{\#}P$ and $Q':=S_{\#}Q$ for any probability measures $P,Q$ with finite first moments. Let $\nu=\mathcal N(0,(1-\bar\alpha_t)I)$. Then
\[
W_1(P*\nu,\;Q*\nu)\;\le\; W_1(P',Q') \;=\; \sqrt{\bar\alpha_t}\,W_1(P,Q).
\]
\end{theorem}
\end{shaded}

\begin{proof}
We prove this by coupling. Let $\Pi(P,Q)$ denote the set of couplings of $P$ and $Q$.
For any $\pi\in\Pi(P,Q)$ define $\pi':=(S\times S)_{\#}\pi$, where $S(x)=\sqrt{\bar\alpha_t}\,x$.
Then $\pi'\in\Pi(P',Q')$ and
\begin{align} 
\int \|x'-y'\|\,\mathrm d\pi'(x',y')
& = \int \|S(x)-S(y)\|\,\mathrm d\pi(x,y) \nonumber \\
& = \sqrt{\bar\alpha_t}\int \|x-y\|\,\mathrm d\pi(x,y).
\end{align}
Since $W_1(P',Q')$ is the infimum of $\int\|x'-y'\|\,\mathrm d\tilde\pi$ over all $\tilde\pi\in\Pi(P',Q')$, we have
\begin{align} 
W_1(P',Q') \le \sqrt{\bar\alpha_t}\int \|x-y\|\,\mathrm d\pi(x,y)
\end{align}
for every \(\pi\in\Pi(P,Q)\). Taking the infimum over \(\pi\in\Pi(P,Q)\) yields
\begin{align}
\label{eq:proof_W_contraction_step1}
W_1(P',Q') \le \sqrt{\bar\alpha_t}\,W_1(P,Q).
\end{align} 
%By choosing an optimal coupling for $W_1(P,Q)$ and pushing it forward we obtain equality, so $W_1(P',Q')=\sqrt{\bar\alpha_t}W_1(P,Q)$.
Next, let $(U,V)$ be any coupling of $P'$ and $Q'$. Independently draw $Z\sim\nu$ and consider $(U+Z,V+Z)$. This is a coupling of $P'*\nu$ and $Q'*\nu$. For this coupling
\[
E\big[\|(U+Z)-(V+Z)\|\big] = E\big[\|U-V\|\big].
\]
Taking the infimum over couplings $(U,V)$ gives
\begin{align}
\label{eq:proof_W_contraction_step2}
W_1(P'*\nu,\;Q'*\nu) \le W_1(P',Q').
\end{align}

% Finally, we combine Eq.~\ref{eq:proof_W_contraction_step1}
% and Eq.~\ref{eq:proof_W_contraction_step2}. Hence
% \begin{align}
% W_1(P*\tilde\nu,\;Q*\tilde\nu)
% & = W_1\bigl((S^{-1})_{\#}(P')*\nu,\; (S^{-1})_{\#}(Q')*\nu\bigr) \nonumber \\
% & \le W_1(P',Q') = \sqrt{\bar\alpha_t}\,W_1(P,Q), \nonumber
% \end{align}
% which, when applied to $P_{tr},P_{val}$ and their diffused versions, yields
% \[
% W_1(P_{tr,t},P_{val,t}) \le \sqrt{\bar\alpha_t}\,W_1(P_{tr},P_{val}).
% \]
% This completes the proof. The proof requires only the existence of first moments to ensure the expectations are finite.

Finally, combining Eq.~\ref{eq:proof_W_contraction_step1}
and Eq.~\ref{eq:proof_W_contraction_step2}, and noting that
\[
P_t = (S)_{\#}P * \nu,\qquad Q_t = (S)_{\#}Q * \nu,
\]
we obtain
\begin{align}
W_1(P_t,Q_t)
&= W_1(P'*\nu,\;Q'*\nu) \nonumber \\
&\le W_1(P',Q') \nonumber \\
&\le \sqrt{\bar\alpha_t}\,W_1(P,Q).
\end{align}
This completes the proof.
\end{proof}
In Theorem \ref{thm:main}, we treat the diffusion discriminator as operating on a lifted input space,
and compare bounds under the assumption of comparable effective capacity
and Lipschitz control.

\begin{shaded}
    % \begin{theorem}[Generalization Error Contraction]\label{thm:main}
    \begin{theorem}[Train-validation Generalization Error Contraction]\label{thm:main}
Let $D_{adv}$ and $D_{diff}$ be the well-trained standard discriminator (trained on clean data) and the diffusion-based discriminator (trained on noisy data), respectively.
Then,  the validation generalization gaps satisfy
\begin{align}
\label{eq:generalization_theory_eq1}
\bigl|E_{x\sim P_{tr}}D_{adv}(x)-E_{x\sim P_{val}}D_{adv}(x)\bigr|
\le W_1(P_{tr},P_{val}),
\end{align}
and
\begin{align}
& \bigl|E_{x\sim P_{tr,t}}D_{diff}(x)-E_{x\sim P_{val,t}}D_{diff}(x)\bigr|  \nonumber \\
& \le W_1(P_{tr,t},P_{val,t}) \nonumber \\
& \le \sqrt{\bar\alpha_t}\,W_1(P_{tr},P_{val}).
\end{align}
That is, for $\bar\alpha_t<1$, the diffusion-trained discriminator $D_{diff}$ has a strictly tighter upper bound on its validation generalization gap than $D_{adv}$.
\end{theorem}
\end{shaded}

\begin{proof}
% Suppose
%\begin{enumerate}
  %\item each discriminator is (approximately) $1$-Lipschitz (WGAN-GP's gradient penalty enforces this approximately);
  %\item each discriminator is trained sufficiently on its corresponding training set so that training error is negligible;
  %\item $P_{tr},P_{val}$ have finite first moments.
%\end{enumerate}

The first inequality (Eq.~\ref{eq:generalization_theory_eq1}) follows directly from Lemma~\ref{lem:KR} since $D_{adv}$ is $1$-Lipschitz:
\begin{align}
& \bigl|E_{x\sim P_{tr}}D_{adv}(x)-E_{x\sim P_{val}}D_{adv}(x)\bigr| \nonumber \\
& \le \sup_{\|f\|_{\mathrm{Lip}}\le1}\{E_{P_{tr}}f-E_{P_{val}}f\} = W_1(P_{tr},P_{val}).
\end{align}
Likewise, for $D_{diff}$,
\[
\bigl|E_{x\sim P_{tr,t}}D_{diff}(x)-E_{x\sim P_{val,t}}D_{diff}(x)\bigr|
\le W_1(P_{tr,t},P_{val,t}),
\]
and applying Theorem~\ref{thm:W-contraction} yields the stated contraction bound. This completes the proof.
\end{proof}

\begin{remark}
From the theorem, we have:
\begin{itemize}
  \item \textbf{Negligible Training Error Assumption.} In practice, discriminators do not reach exact optima. If the discriminator leaves a residual training error $\varepsilon_{\mathrm{train}}$ (i.e.\ the empirical objective differs from its ideal by $\varepsilon_{\mathrm{train}}$), then the validation error bounds above should be augmented by $\varepsilon_{\mathrm{train}}$. Concretely, if the empirical training expectation differs from the population training expectation by at most $\varepsilon_{\mathrm{train}}$, the validation bound for $D_{diff}$ becomes
  \begin{align}
  \bigl|E_{P_{val,t}}[D_{diff}] &- |E_{P_{tr,t}}[D_{diff}]\bigr| \\
  & \le \varepsilon_{\mathrm{train}} + \sqrt{\bar\alpha_t}\,W_1(P_{tr},P_{val}).
  \end{align}
  \item \textbf{Effect of Gradient Penalty.} WGAN-GP adds a penalty term of the form $\lambda E[(\|\nabla_x D(\tilde x)\|-1)^2]$ (with $\tilde x$ sampled on straight-line segments between real and fake samples) to encourage $D$ to be near 1-Lipschitz on the data region. Strictly speaking, this enforces a local Lipschitz condition; in our theoretical statement, we assume the trained discriminator is approximately $L$-Lipschitz on the relevant support. If $L\neq1$ replace $W_1$ by $L\,W_1$ in the inequalities, and the contraction factor becomes $L\sqrt{\bar\alpha_t}$.
\end{itemize}
\end{remark}

\section{Conclusion}
In this paper, we investigate the zero-shot learning (ZSL) problem under varying amounts of training data, revealing that the issue of spurious visual-semantic correlations on both seen and unseen classes is exacerbated when training samples are scarce. To enhance visual-semantic correlation, we introduce ZeroDiff++, which incorporates three key components for generator training: (1) a diffusion forward chain to augment the limited training set; (2) SC-based representations to effectively represent each limited sample; (3) mutually learned discriminators to validate generated features from multiple perspectives, including predefined semantics, contrastive representations, and diffusion processes; and two for feature generation: (4) diffusion-based test-time adaption that adapt generator on unseen classes by denoising pseudo-labeled test features; and (5) diffusion-based test-time generation allowing traceable generation and partial feature imagination via diffused test features. We also theoretically proved that diffusion-based discriminators have a better overfitting reduction effect than traditional clean-sample discriminators. Experiments conducted on three popular datasets demonstrate the data efficiency of our method. Our ZeroDiff++ significantly enhances zero-shot capacity, performing well with both abundant and limited training samples.


% \section*{Acknowledgments}
% This should be a simple paragraph before the References to thank those individuals and institutions who have supported your work on this article.
 
\bibliography{TPAMI_version/ZeroDiff++_TPAMI}
\bibliographystyle{IEEEtran}

\clearpage

% {\appendix[Proof of the Zonklar Equations]
% Use $\backslash${\tt{appendix}} if you have a single appendix:
% Do not use $\backslash${\tt{section}} anymore after $\backslash${\tt{appendix}}, only $\backslash${\tt{section*}}.
% If you have multiple appendixes use $\backslash${\tt{appendices}} then use $\backslash${\tt{section}} to start each appendix.
% You must declare a $\backslash${\tt{section}} before using any $\backslash${\tt{subsection}} or using $\backslash${\tt{label}} ($\backslash${\tt{appendices}} by itself starts a section numbered zero.)}
{
% \appendix[Proof of the Zonklar Equations]

\appendices

\section{SC-based Representation}
\label{app:sc2cs}

We use t-SNE to visualize the CE-based features and SC-based representations in Fig.~\ref{fig:tsne_cs_sc_all}.
We can find every class cluster more tightly in CE-based space than in SC-based space.
It verifies previous work claims SC-based representations have larger intra-class variation.
It also means that SC-based representations contain more intra-class uncertainty.
To further verify this point, we visualize the class `fox' in Fig.~\ref{fig:tsne_cs_sc_fox}.
We can find that all instances are grouped into a single cluster in CE-based space, while different sub-classes of fox (i.e., red fox, white fox, and grey fox) could be separately grouped in SC-based space.
In other words, SC-based representations could reflect the characteristics of every instance better than CE-based features.

\begin{figure}
    \centering
    \includegraphics[width=\linewidth]{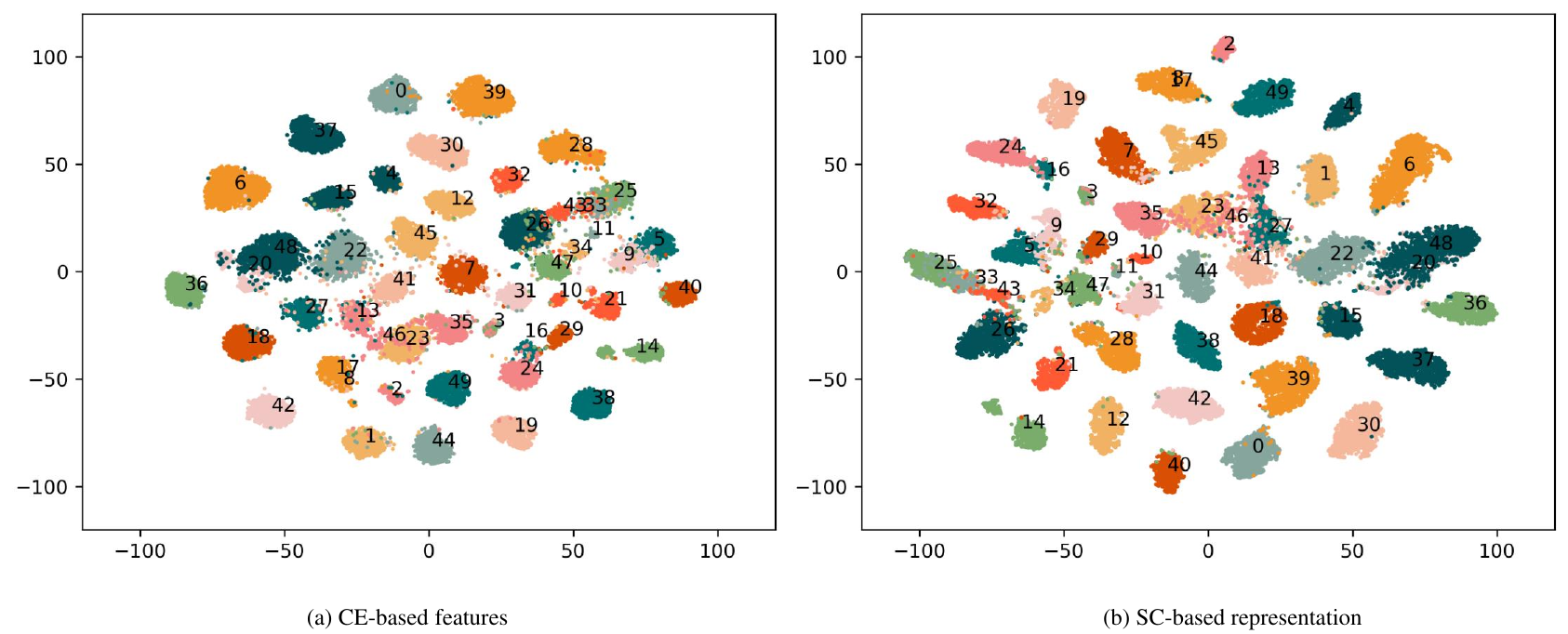}
    \caption{Comparison of CS-based features and SC-based representations with t-SNE visualization for all classes on AWA2. We can find that SC-based representations have larger intra-class variation than CE-based features.
}
    \label{fig:tsne_cs_sc_all}
\end{figure}

\begin{figure}
    \centering
    \includegraphics[width=\linewidth]{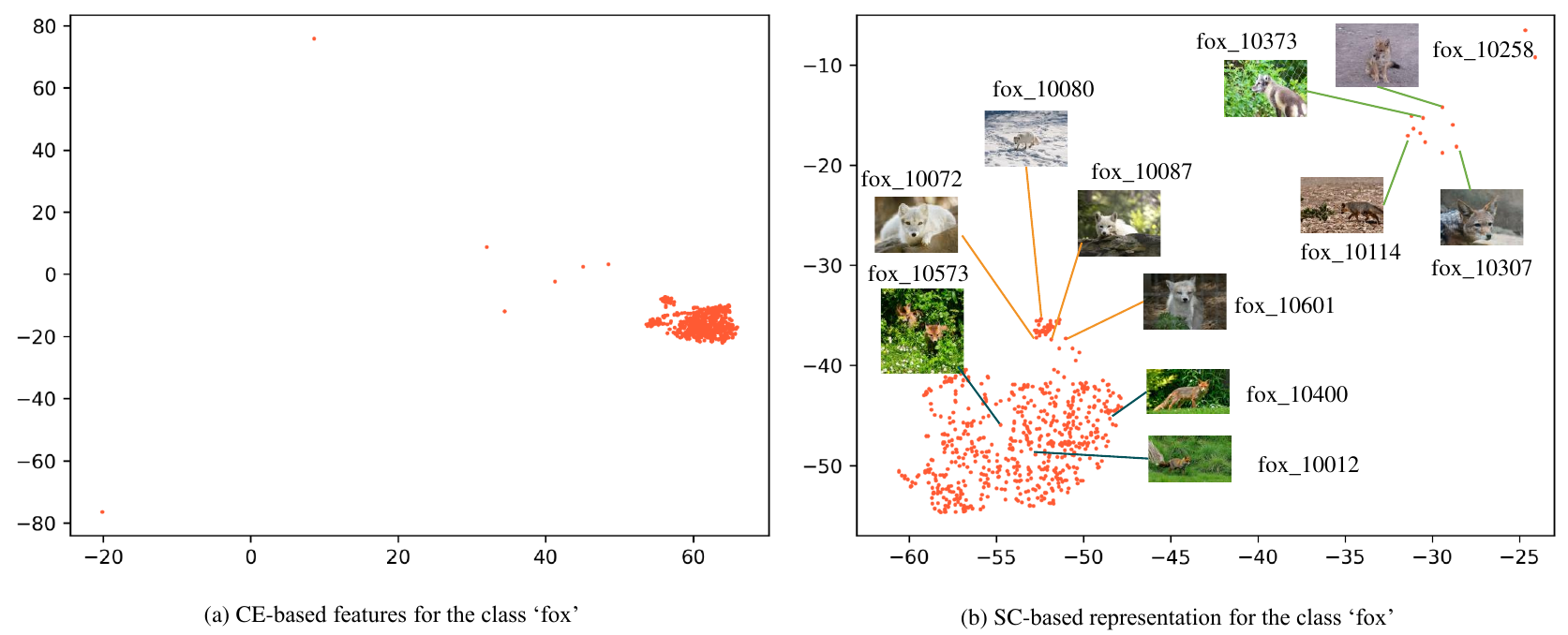}
    \caption{Comparison of CS-based features and SC-based representations with t-SNE visualization for the class `fox' on AWA2. We can find all instances of fox are clustered in a group in CE-based space while different sub-classes (i.e. red fox, white fox and grey fox) are gathered in different groups in SC-based space.
}
    \label{fig:tsne_cs_sc_fox}
\end{figure}

\section{Mutual Learning Effectiveness}
\label{app:mu_effect}
To further verify the effect of our \( \mathcal{L}_{mu} \), we designed an additional experiment to show the change in critic score.
\begin{enumerate}
    \item As shown in Fig.~\ref{fig:CS_curve} (a), we display the epoch-$\Delta^{s}_{adv}$ curve of \( D_{adv} \) (Eq.~\ref{eq:delta_adv_seen}).
    With \( \mathcal{L}_{mu} \), \( \Delta^{s}_{adv} \) decreases significantly compared to the counterpart without \( \mathcal{L}_{mu} \), suggesting that knowledge from \( D_{diff} \) helps \( D_{adv} \) reduces over-fitting in the training set.
    \item We also draw the the curve of epoch-$\Delta^{u}_{adv}$ (Eq.~\ref{eq:delta_adv_unseen}) in Fig.~\ref{fig:CS_curve} (c), showing that the potential over-fitting on unseen classes consistently reduced by our \( \mathcal{L}_{mu} \).
    \item Conversely, the knowledge from \( D_{adv} \) can also benefit \( D_{diff} \). As shown in Fig.~\ref{fig:CS_curve} (b), we display the critic score difference of \( D_{diff} \) between noised real training features \( \mathbf{v}_{t} \) and noised fake training features \( \tilde{\mathbf{v}}_{t} \), i.e., 
    \begin{align}
    \label{eq:delta_diff}
        \Delta_{diff}(\mathbf{v}_{t}, \tilde{\mathbf{v}}_{t}) &= D_{diff}(\mathbf{v}_{t},\mathbf{v}_{t+1},\mathbf{r}_{0},\mathbf{a},t) \\ \nonumber
        & - D_{diff}(\tilde{\mathbf{v}}_{t},\mathbf{v}_{t+1},\mathbf{r}_{0},\mathbf{a},t). 
    \end{align}
    After adding \( \mathcal{L}_{mu} \), \( \Delta_{diff} \) increases significantly. This indicates that the distinguishing ability of \( D_{diff} \) is enhanced by the knowledge from \( D_{adv} \).
\end{enumerate}

\begin{figure}
    \centering
    \includegraphics[width=0.95\linewidth]{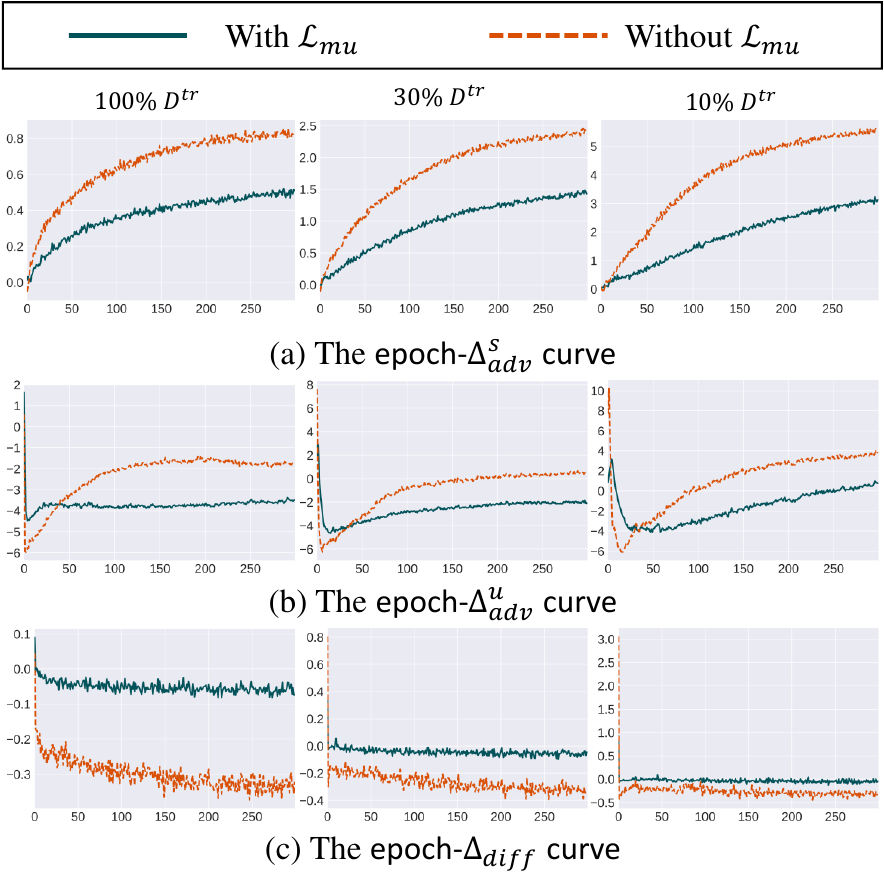}
    \caption{ The effect of $\mathcal{L}_{mu}$ to $\Delta^{s}_{adv}$ (Eq.~\ref{eq:delta_adv_seen}), $\Delta^{u}_{adv}$ (Eq.~\ref{eq:delta_adv_unseen}) and $\Delta_{diff}$ (Eq.~\ref{eq:delta_diff}) on AWA2. (a) indicates that our \( \mathcal{L}_{mu} \) mitigates the overfitting of \( D_{adv} \) in the training set. (b) exhibits that our  \( \mathcal{L}_{mu} \) enhances the correlation on unseen classes. (c) shows that the distinguishing ability of \( D_{diff} \) is enhanced by our \( \mathcal{L}_{mu} \).
 }
    \label{fig:CS_curve}
\end{figure}

}

\vfill

\end{document}